\documentclass[12pt]{article}
\makeatletter
\def\input@path{{./}{../}}
\makeatother
\usepackage[margin=1in]{geometry}
\usepackage[utf8]{inputenc}
\usepackage[T1]{fontenc}
\usepackage{lmodern}
\usepackage{natbib}
\setcitestyle{authoryear}
\usepackage{amsmath, amssymb}
\usepackage{graphicx}
\usepackage{float}
\usepackage{url}
\usepackage[colorlinks=true, linkcolor=blue, urlcolor=blue, citecolor=blue,
  hypertexnames=false, hyperfootnotes=false]{hyperref}
\makeatletter
\renewcommand{\NAT@open}{(}
\renewcommand{\NAT@close}{)}
\renewcommand{\hyper@natlinkbreak}[2]{#1}
\makeatother
\usepackage{amsfonts,amsmath,amssymb,amsthm,amscd} %
\usepackage{mathtools}
\mathtoolsset{showonlyrefs} %
\usepackage{latexsym,bbold}
\usepackage{fontawesome}
\usepackage{hyperref}

\usepackage{setspace}
\usepackage{framed}
\usepackage{graphicx}
\graphicspath{ {images/} }

\usepackage[inline]{enumitem}

\usepackage{booktabs}

\usepackage{caption}
\usepackage{subcaption}

\usepackage{listings}
\usepackage{xcolor}

\definecolor{codegreen}{rgb}{0,0.6,0}
\definecolor{codegray}{rgb}{0.5,0.5,0.5}
\definecolor{codepurple}{rgb}{0.58,0,0.82}
\definecolor{backcolour}{rgb}{0.95,0.95,0.92}

\lstdefinestyle{mystyle}{
    backgroundcolor=\color{backcolour},   
    commentstyle=\color{codegreen},
    keywordstyle=\color{magenta},
    numberstyle=\tiny\color{codegray},
    stringstyle=\color{codepurple},
    basicstyle=\ttfamily\footnotesize,
    breakatwhitespace=true,         
    breaklines=true,                 
    captionpos=b,                    
    keepspaces=true,                 
    numbers=left,                    
    numbersep=5pt,                  
    showspaces=false,                
    showstringspaces=false,
    showtabs=true,                  
    tabsize=2
}

\allowdisplaybreaks[1]

\newtheorem{theorem}{Theorem}[section]
\newtheorem{proposition}{Proposition}[section]
\newtheorem{lemma}{Lemma}[section]
\newtheorem{corollary}{Corollary}[section]

\newtheorem{assumption}{Assumption}[section]
\theoremstyle{definition}
\newtheorem{defn}{Definition}[section]
\newtheorem{example}{Example}[section]

\theoremstyle{remark}
\newtheorem*{remark}{Remark}

\DeclareMathOperator*{\argmax}{arg\,max}
\DeclareMathOperator*{\argmin}{arg\,min}
\newcommand{\df}{\operatorname{df}}
\newcommand{\Cov}{\operatorname{Cov}}

\newcommand{\nc}{\newcommand}
\nc{\on}{\operatorname}
\nc{\Spec}{\on{Spec}}

\DeclarePairedDelimiter{\paren}{(}{)}

\newcommand{\Z}{\mathbb{Z}}
\newcommand{\R}{\mathbb{R}}

\newcommand{\E}{\mathbb{E}}

\newcommand{\activeset}{\mathcal{M}}
\newcommand{\allowedset}{\mathcal{V}}

\newcommand{\bX}{\mathbf{X}}
\newcommand{\bbf}{\mathbf{f}}
\newcommand{\br}{\mathbf{r}}
\newcommand{\bx}{\mathbf{x}}
\newcommand{\bz}{\mathbf{z}}
\newcommand{\by}{\mathbf{y}}

\newcommand{\bI}{\mathbf{I}}

\newcommand{\bP}{\mathbf{P}}

\newcommand{\bbeta}{\boldsymbol{\beta}}

\newcommand{\bSigma}{\boldsymbol{\Sigma}}

\newcommand{\weight}[2]{w_{#1}^{#2}}
\newcommand{\limitweight}[2]{\tilde{w}_{#1}^{#2}}
\newcommand{\llweight}[2]{\check{w}_{#1}^{#2}}

\graphicspath{{./}{../}{images/}{../images/}}
\title{Understanding Randomization in Greedy Model Search}
\author{
  Xin Chen\thanks{Department of Operations Research and Financial Engineering, Princeton University, Sherrerd Hall, Princeton, NJ 08544. Email: \texttt{xc5557@princeton.edu}}
  \and Jason M. Klusowski\thanks{Department of Operations Research and Financial Engineering, Princeton University, Sherrerd Hall, Princeton, NJ 08544. Email: \texttt{jason.klusowski@princeton.edu}}
  \and Yan Shuo Tan\thanks{Department of Statistics and Data Science, National University of Singapore, Block S16, Level 7, 6 Science Drive 2, Singapore 117546. Email: \texttt{yanshuo@nus.edu.sg}}
  \and Chang Yu\thanks{Department of Operations Research and Financial Engineering, Princeton University, Sherrerd Hall, Princeton, NJ 08544. Email: \texttt{cy7438@princeton.edu}}
}
\date{}
\newenvironment{keywords}{\vspace{1em}\noindent\textbf{Keywords:} }{}
\newenvironment{acknowledgements}{\section*{Acknowledgements}}{}
\begin{document}
\maketitle
\begin{abstract}
We study feature subsampling in greedy model search, a randomization mechanism motivated by random forests and other ensemble methods. Although recent theory draws an analogy between this randomization and ridge regularization, these results largely rely on base learners optimized via ordinary least squares (OLS). We therefore investigate the effects of feature subsampling on greedy forward selection, a tractable abstraction of the adaptive split search used by decision trees. Assuming an orthogonal design, we prove that, under suitable conditions, ensembling with a fixed feature subsampling size can reduce expected training error without increasing degrees of freedom, thereby reducing fixed design prediction error. We also show that both the training error and degrees of freedom need not be monotone in the subsampling rate, breaking the analogy with standard shrinkage methods like the lasso or ridge regression under orthogonality. Furthermore, we characterize the exact asymptotic behavior of the estimator, showing that it adaptively reweights OLS coefficients based on their rank, with weights that are bounded by explicit logistic functions.

\end{abstract}
\section{Introduction}

Introduced by \citet{breiman2001random}, random forests (RFs) are ensembles of decision trees, each grown via greedy, recursive partitioning of the input space \citep{breiman1984classification}.
Each tree is independently randomized in two ways. First, it is trained on a bootstrap sample of the data. Second, at every internal node it considers only a randomly drawn subset of $m$ (out of $p$ total) features when searching for the next split.

There has been significant theoretical work on the consistency of RFs, especially in high dimensions~\citep{scornet2015consistency,klusowski2020sparse,syrgkanis2020estimation,chi2022asymptotic,klusowski2024large,mazumder2024convergence}, but there is not yet a complete consensus on \emph{how randomization improves prediction performance}.
\citet{breiman2001random} originally explained the benefits of randomization as increasing diversity among the trees, thereby reducing the variance in their predictions.
Interestingly, in his concluding remarks, he speculated that randomization may ``act to reduce bias'' but that ``the mechanism for this is not obvious''.
These two explanations have had divergent fates, with \citet{hastie2009elements} arguing in their influential textbook that ``the improvements in prediction obtained by bagging or random forests are solely a result of variance reduction''.

Seminal work by \citet{mentch2020randomization} offered a quantitative analysis of the role played by the feature subsampling ratio $\gamma = m/p$.
In their numerical experiments, they found that decreasing $\gamma$ decreases the degrees of freedom in RFs.
To further explore the generality of this phenomenon, they applied ensembling with feature subsampling to the greedy forward selection (FS) algorithm~\citep{draper1966applied}, finding that decreasing $\gamma$ also decreases the degrees of freedom of the resulting family of estimators.
Finally, in work concurrent with that of \citet{lejeune2020implicit}, they analyzed ensembling with feature subsampling on ordinary least squares (OLS) estimators. Such an ensemble is linear in the response, whereas ensembles of FS generally are not because the selected features depend on the response. Assuming orthonormal features, they found that the resulting regression coefficients were shrunk from the OLS values by a factor of $\gamma$, making the algorithm equivalent to ridge regression.
They hence concluded that feature subsampling ``serves as a form of implicit regularization'' and that $\gamma$ ``serves much the same purpose as the shrinkage penalty in explicitly regularized regression procedures like lasso and ridge regression''.
Under orthogonality and without bootstrap resampling, they also noted that averaging scales each OLS coefficient by its selection frequency.

Recent work building on \citet{mentch2020randomization} has proceeded in two directions.
In one direction, \citet{patil2024asymptotically} relaxed the orthonormal feature assumption to one they called ``asymptotic freeness'' (relative to the feature subsampling operator).
Under suitable asymptotic freeness assumptions, they obtained the risk of an equivalent ridge estimator plus a sketching variance term proportional to inverse ensemble size, which vanishes for an infinite ensemble.
Meanwhile, other works studying rowwise sampling (``data subsampling'') established a broad theoretical equivalence between this form of randomized ensembling and ridge penalization \citep{du2023subsample,patil2023generalized}.
Together, these works seem to strengthen the analogy between various forms of subsampling randomization and ridge style shrinkage.

In the other direction, studies by \citet{curth2024random}, \citet{liu2025randomization}, \citet{mei2024exogenous}, and \citet{revelas2025random} suggest that feature subsampling can also reduce bias in RFs via two mechanisms, although their notions of bias differ. The first is enlarging the effective model class \citep{curth2024random}. The second is overcoming failures of greedy optimization \citep{liu2025randomization}.
These results challenge the ``pure variance reduction'' hypothesis of \citet{hastie2009elements} and suggest that \emph{although the shrinkage analogy in \citet{mentch2020randomization} seems appropriate when describing ensembles of base learners fit with convex optimization, this analogy becomes strained when subsampling is interleaved with greedy optimization, as is the case for RFs}.
Existing theory on subsampling based ensembling primarily concerns base learners that are not adaptive or that are analytically tractable because they arise from convex optimization, as well as oracle versions of RFs \citep{mei2024exogenous}.
For settings in which feature subsampling is interleaved with greedy, adaptive fitting, evidence for bias reduction mechanisms, however, remains largely empirical \citep{curth2024random,liu2025randomization,revelas2025random}.

The goal of our paper is therefore to isolate the interaction between feature subsampling and greedy, response dependent selection.
Inspired by \citet{mentch2020randomization}, we study their ensembling strategy as applied to FS rather than decision trees as in RFs.
The connection between FS and regression trees is closer than one might first expect. A regression tree can be written as a linear combination of orthogonal stump functions \citep[Lemma~2.1]{cattaneo2024convergence} and splitting a terminal node selects the available stump most correlated with the residual.
This is precisely the same logic as FS, except that the dictionary of stumps is not fixed but depends on the current partition of the input space.

We assume a fixed design setting and orthogonal features, in which case FS is equivalent to best subset selection (BSS, \citealt{hastie2020extended}).
Letting FS($k$) and EFS($k,m$) refer to forward selection after $k$ steps and its ensemble, respectively, we obtain the following theoretical results.
\begin{itemize}
    \item Under a condition on the expected squared coefficients and for $p$ large enough, a fixed $m<p$ improves expected training error over FS($k$).
    \item If the degrees of freedom of FS($k$) form a concave sequence in $k$, then the degrees of freedom of EFS($k,m$) are nondecreasing in $m$.
    \item If the concavity condition does not hold, then the degrees of freedom of EFS($k,m$) need not be monotone in $m$, and can exceed those of FS($k$) at intermediate values.
\end{itemize}
Together, the training error and concavity conditions imply that a single deterministic choice of $m<p$ gives smaller expected training error and no larger degrees of freedom than FS($k$), and hence smaller fixed design prediction error.
On the other hand, degrees of freedom reduction is not guaranteed but depends on certain structural assumptions.
Furthermore, in both ridge and lasso regression under orthogonality, the expected training error (resp.\ degrees of freedom) is nondecreasing (resp.\ nonincreasing) in the regularization parameter $\lambda$~\citep{zou2007degrees,hastie2009elements}, yielding a monotone tradeoff between the two quantities.
Since this monotonic relationship does not hold for $m$ in EFS($k,m$), we see that the analogy with ridge and lasso does not fully describe the role of randomization in this setting.

To prove these results, we characterize how feature subsampling smooths the hard thresholding of FS(\(k\)) through adaptive reweighting of OLS coefficients. These weights depend only on the ranks of the observed OLS coefficients and not their values. We establish an exact finite recurrence for the selection probabilities, which, in the asymptotic regime \(m,p\to\infty\) with \(m/p\to\gamma\in(0,1)\), yields limiting weights bounded by explicit logistic functions, with growth rate depending on \(\gamma\) and a transition near the selection cutoff \(k\) when \(k\) is large. The same weight representation lets us compare training errors coordinate by coordinate and use majorization to compare degrees of freedom under the concavity condition. 

\section{Preliminaries}

\subsection{Setup and notation}
Assume a fixed design regression model $y_i = f(\bx_i) + \varepsilon_i$, where $f$ is an arbitrary regression function and $\varepsilon_1,\ldots,\varepsilon_n$ are independent and identically distributed (IID) noise variables with mean zero and variance $\sigma^2>0$. Noiseless constructions and limits are stated explicitly.
We assume $\bX = [\bx_1, \bx_2,\ldots,\bx_n]^\top \in \R^{n \times p}$ is a fixed design matrix and write $\by=(y_1,\ldots,y_n)^\top$.
For vectors $\mathbf{a}, \mathbf{b} \in \mathbb{R}^n$, we define the inner product as $
\langle \mathbf{a}, \mathbf{b} \rangle = \mathbf{a}^{\top}\mathbf{b}/n
$, with associated norm $\|\mathbf{a}\|^2=\langle\mathbf{a},\mathbf{a}\rangle$. The notation $\|\cdot\|_{\ell_2}$ instead denotes the ordinary Euclidean norm.
For each $j \in [p]=\{1,\ldots,p\}$, let $\bx_{\cdot j,p} = (x_{1j}, x_{2j}, \dots, x_{nj})^{\top}$ denote the $j$-th feature vector, which we normalize in the theoretical results so that $\|\bx_{\cdot j,p}\| = 1$.
We use regular font to denote functions such as the coordinate functions $x_{j,p}$.

\subsection{Algorithms}

\paragraph{Forward selection.}
Forward selection (FS) adds one feature at each step, choosing the feature that most reduces training error after refitting all selected coefficients by least squares.
Starting with an empty active set $\activeset_0 = \emptyset$ and $\hat{\bbf}_{0,p}=\boldsymbol{0}$, it proceeds for $l = 1,2,\dots,k$, where $0\leq k\leq p$, by updating $\activeset_{l-1}$ with the index
\begin{equation}\label{eq:fs_objective}
j_l = \argmin_{j \in \allowedset_l}
\|\by - \bP_{\activeset_{l-1} \cup \{j\}}\by\|
= \argmax_{j \in \allowedset_l}
\frac{|\langle \br_{l-1}, \bx_{\cdot j,p} \rangle|}
     {\|\bP_{\activeset_{l-1}}^{\perp}\bx_{\cdot j,p}\|},
\end{equation}
where $\allowedset_l = [p]\setminus \activeset_{l-1}$ is the set of eligible indices.
Here, $\bP_{\activeset_l}$ is the projection matrix onto the column span of the selected features, and $\bP_{\activeset_l}^{\perp}$ projects onto its orthogonal complement.
We set
\[
\hat{\bbf}_{l,p} = \bP_{\activeset_{l}}\by.
\]
The residual vector at step $l$ is $\br_l = \by - \hat{\bbf}_{l,p}$.
A zero denominator in~\eqref{eq:fs_objective} is assigned score zero, since that feature adds no direction to the selected span. Ties are resolved by a fixed ordering of the original feature indices.
After $k$ steps, the resulting predictive function $\hat{f}_{k,p}$ is referred to as the FS($k$) estimator.

\paragraph{Ensemble forward selection.}
Following \citet{mentch2020randomization}, we introduce randomization into each FS($k$) step by modifying~\eqref{eq:fs_objective} so that the maximization is taken over a randomly drawn subset.
Only candidate features are randomized, and every base estimator uses the same observations without bootstrap sampling.
Specifically, at step $l$, we sample a fresh candidate set $\allowedset_l \subseteq [p]\setminus \activeset_{l-1}$ of size $m$ uniformly without replacement, where $1\leq m\leq p$. If $|[p]\setminus \activeset_{l-1}| < m$, we use all remaining features without randomization.
The update index $j_l$ is then chosen as in~\eqref{eq:fs_objective}, but restricted to $j \in \allowedset_l$.
Stopping after $k$ steps yields a randomized FS($k$) estimator $\check{f}_{k,m,p,b}$ for replicate $b$.
Repeating this process independently for $b=1,\ldots,B$ and aggregating via simple averaging yields the ensemble estimator
\[
\hat f_{k,m,p,B}=\frac{1}{B}\sum_{b=1}^B\check f_{k,m,p,b}.
\]
By the strong law of large numbers, conditional on the data, the estimator $\hat f_{k,m,p,B}$ converges almost surely to its expected value as $B\to\infty$, with the expectation taken over the randomness of the subset choices.
For the rest of this paper, we refer to this limit as the EFS($k,m$) estimator, denoting it as $\hat f_{k,m,p}$.
Appendix~\ref{subsec:finite_ensemble} bounds the conditional mean squared error of finite ensemble averaging at order $1/B$, without requiring orthogonality.

\subsection{Prediction error decomposition}

In a fixed design setting, the complexity of an estimator $\hat f$ is often quantified via its \emph{degrees of freedom} \citep{hastie2009elements},
\begin{equation}
    \df(\hat f) = \sum_{i=1}^n \frac{\Cov( y_i,\hat f(\bx_i))}{\sigma^2}.
\end{equation}
Letting $\bbf = (f(\bx_1),\ldots,f(\bx_n))^\top$ and $\hat\bbf = (\hat f(\bx_1),\ldots,\hat f(\bx_n))^\top$ denote the true and fitted function values respectively, the expected mean squared prediction error for estimating $\bbf$ can be expressed in terms of expected training error and degrees of freedom through the optimism identity\footnote{Although this is similar in spirit to a bias-variance decomposition, the expected training error is not classical squared bias, and degrees of freedom is not estimator variance.}
\begin{equation}
\label{eq:expected_error}
    \E\big[\Vert \mathbf{f}- \hat{\mathbf{f}}\Vert^2\big] + \sigma^2 = \E\big[\Vert \mathbf{y}- \hat{\mathbf{f}}\Vert^2\big] + \frac{2\sigma^2}{n}\df(\hat{f}).
\end{equation}
We will therefore analyze how EFS($k,m$) compares to FS($k$) on each term on the right hand side of \eqref{eq:expected_error}.
Our comparisons concern fixed values of $k$ and $m$, but if either is selected using the response, the degrees of freedom must also account for this selection.

\section{Properties of \texorpdfstring{EFS($k,m$)}{EFS(k,m)} weights}
\label{sec:ensemble_weight_properties}

\subsection{\texorpdfstring{EFS($k,m$)}{EFS(k,m)} as reweighted OLS}

We assume orthogonal features for the remaining results, except in Section~\ref{sec:local_minima}, its supplementary analyses, and the simulations.

\begin{assumption}[Orthogonality] \label{ass:orthogonal} 
The design matrix satisfies $\widehat{\boldsymbol{\Sigma}} =\bX^{\top}\bX/n = \mathbf{I}_p$. In particular, $p\leq n$.
\end{assumption}

Extending the results to general correlated designs is an important direction, but would require controlling how correlations perturb the adaptive ordering induced by the greedy procedure.
With correlated features, each selection can change the remaining residual correlations and projection denominators, making selection probabilities depend on the Gram matrix and preceding selections.
Following established practice, we therefore focus on the orthogonal case, which allows us to derive exact asymptotic expressions for interpretable coefficient behavior and precisely characterize how the randomization from feature subsampling propagates through the FS($k$) algorithm.
This also enables direct comparisons with the broader regularization landscape, including ridge, lasso \citep{tibshirani1996regression}, elastic net \citep{zou2005regularization}, LARS \citep{efron2004least}, SCAD \citep{fan2001variable}, MCP \citep{zhang2010nearly}, and SureShrink \citep{donoho1995adapting}, since all of them have been studied in the orthogonal setting.

Let $\hat f_{\operatorname{OLS}} = \sum_{j=1}^p \hat\beta_{j, p} x_{j,p}$ denote the ordinary least squares (OLS) projection of the response onto the full dictionary.
Under Assumption \ref{ass:orthogonal}, its coefficients can be expressed as inner products $\hat \beta_{j,p} = \langle \mathbf{y}, \bx_{\cdot j,p}\rangle $.
The signal projection coefficients are $\beta_{j,p}=\langle\bbf,\bx_{\cdot j,p}\rangle$.
It turns out that both the FS($k$) and EFS($k,m$) estimators can be written in terms of the OLS coefficients.
For ease of notation, we establish an ordering convention for the features $x_{1,p},x_{2,p},\ldots,x_{p,p}$ based on these coefficients.

\begin{defn}[Feature ordering]
\label{def:betaj}
Let $\pi$ be the permutation that sorts the OLS coefficients by magnitude, using the same tie rule as FS. Write $\hat\beta_{(j),p}=\hat\beta_{\pi(j),p}$, $\beta_{(j),p}=\beta_{\pi(j),p}$, and $x_{(j),p}=x_{\pi(j),p}$, so that
\begin{equation} \label{eq:ordering}
    |\hat\beta_{(1),p}| \geq |\hat\beta_{(2),p}| \geq \cdots \geq |\hat\beta_{(p),p}|.
\end{equation}
\end{defn}
Note that this ordering is random as it depends on the response noise. For coefficients and features, subscripts without parentheses denote original feature indices, whereas parenthesized subscripts denote ranks in the ordering~\eqref{eq:ordering}. Under orthogonality, ties have probability zero whenever the noise has a joint density. The fixed tie rule also defines the estimators otherwise.

Using Definition \ref{def:betaj} together with 
the equivalence of FS and BSS, we can write the FS($k$) estimator as
\begin{equation}
    \label{eq:decomp_fs}
\hat f_{k,p} = \sum_{j=1}^p \mathbf{1}(j \leq k)\hat\beta_{(j),p} x_{(j),p},
\end{equation}
which reflects the familiar fact that, in the orthogonal setting, FS($k$) (and BSS) selects the top $k$ features most correlated with the response.
The indicators $\mathbf{1}(j \leq k)$ can be thought of as weights that modulate the OLS solution.
Similarly, the ensemble estimator after $k$ steps has the form
\begin{equation}
    \label{eq:decomp_ensemble}
    \hat f_{k,m,p} = \sum_{j=1}^p \weight{j}{k,m,p}\hat\beta_{(j),p} x_{(j),p},
\end{equation}
where $\weight{j}{k,m,p}=\mathbb{P}(\pi(j)\in\activeset_k\mid\by)$ is the probability that the feature ranked $j$ is selected, with randomness over candidate sampling.
The weight at each rank depends only on $(j,k,m,p)$, but its assignment to a particular feature depends on the observed coefficient ordering.
The asymptotic analysis in Section~\ref{subsec:logistic_approximation} shows how these weights smooth the hard thresholding of FS($k$).

This smoothing is related to the analysis of bagging and subagging by \citet{buhlmann2002analyzing}, where observation resampling and averaging smooth unstable predictors such as decision stumps.
In EFS, feature subsampling generates random inclusion indicators, and ensemble averaging replaces them by their selection probabilities.
When $m=p$, every base learner equals FS($k$), so averaging alone leaves its hard threshold unchanged.

\subsection{Basic properties}

The weight sequence takes values in $[0,1]$ and sums to $k$, since each base learner selects exactly $k$ features.
Unfortunately, the exact values of the weights involve complicated combinatorics and do not have simple closed form expressions in general.
Nonetheless, one can show that the sequence is weakly decreasing and, more importantly, derive a recurrence relation in terms of $j$, $k$, and $p$, which serves as our main technical tool for the subsequent results.
We state and prove these technical facts as Lemmas~\ref{lem:weights_monotonicity} and~\ref{prop:finitep}.

Comparisons at the same $k$ match the number of greedy steps per base learner, but averaging can give EFS more than $k$ nonzero coefficients.
For example, when $m=1$ and $1\leq k\leq p$ is fixed,
\[
\hat f_{k,1,p}=\frac{k}{p}\hat f_{\operatorname{OLS}},
\qquad \df(\hat f_{k,1,p})=k.
\]
Thus the ensemble can be dense even when its degrees of freedom equal $k$.
Our comparisons allow this larger support and do not establish an advantage under a constraint on the total number of selected features.

\subsection{Logistic approximation}
\label{subsec:logistic_approximation}
The finite sample recurrence is three dimensional (in $j$, $k$, and $p$) and inhomogeneous, making it challenging to find a simple closed form solution.
However, in the scaling limit as $m,p\to\infty$ with $m/p\to\gamma\in(0,1)$, the recurrence converges to a two dimensional system that decouples into a one dimensional subsystem with an explicit analytical solution.

\begin{theorem}[Asymptotic weights for fixed $k$]
    \label{thm:asymp}
    Fix positive integers $j,k$ and $\gamma\in(0,1)$, and set $\alpha=-\log(1-\gamma)$. The limiting weights $\limitweight{j}{k,\gamma} =  \lim\limits_{\substack{m \to \infty,\, p \to \infty, \\ m/p \to \gamma}} \weight{j}{k,m,p}$
    \begin{enumerate}[label=(\alph*)]
        \item \label{thm:asymp:a} exist and equal
        \begin{equation}
        \label{eq:closed_wj}
            \sum_{i=1}^{k} (-1)^{k-i}\prod_{l=i}^k\big(e^{-\alpha(j-l)}-e^{-\alpha j}\big),
        \end{equation}
\item \label{thm:asymp:b} satisfy the bounds
\begin{equation}
\label{eq:bounds_wj}
    \frac{1}{1+e^{-\alpha\left(h(\alpha,k)-j\right)}} \leq \limitweight{j}{k,\gamma} \leq \frac{1}{1+e^{-\alpha\left(h(\alpha,k)+1-j\right)}},
\end{equation}
where $h(\alpha,k) = \frac{1}{\alpha}\log(e^{\alpha k} - 1)$ and 
$\alpha \in (0, \infty)$ is defined via $\gamma = 1 - e^{-\alpha}$.
\end{enumerate}
\end{theorem} 

The bounds in \eqref{eq:bounds_wj} reveal that the weights are sandwiched between two decreasing logistic functions with midpoints $h(\alpha,k)$ and $h(\alpha,k)+1$ and common growth rate $\alpha$.\footnote{If the $m$ candidates are drawn independently with replacement from the remaining features, the same limiting weight formulas hold with $\alpha$ replaced by $\gamma$. The first step tail probability $(1-j/p)^m$ tends to $e^{-\gamma j}$, giving the limiting rank recurrence with $e^{-\gamma}$ in place of $1-\gamma=e^{-\alpha}$.}
For $k\geq1$ and large $m,p$ with $m/p\approx\gamma$, these bounds motivate the approximation
\begin{equation} \label{eq:logistic}
\weight{j}{k,m,p} \approx \frac{1}{1+e^{-\alpha(h(\alpha,k)+1/2-j)}}.
\end{equation}
When $\gamma$ is small, the growth rate satisfies $\alpha\approx\gamma$.

These approximations clarify how the EFS($k,m$) parameters affect the shape of the weight function.
As $k$ increases, the logistic curve translates rightward. As $\gamma$ increases (larger $m$), the curve becomes steeper, more closely approximating the sharp thresholding of FS($k$).
Figure \ref{fig:ensemble_comparison} illustrates this behavior for $\gamma \in \{0.05, 0.15, 0.333\}$.

\begin{figure}
    \centering
    \includegraphics[width=0.54545\linewidth]{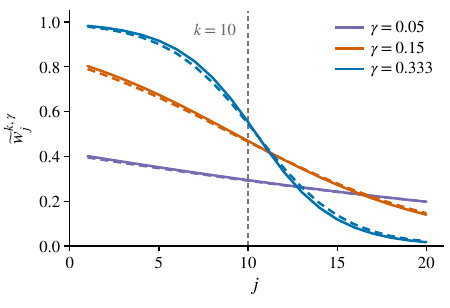}
    \caption{Comparison of $\limitweight{j}{k,\gamma}$ (solid)
    and the logistic approximation \eqref{eq:logistic} (dashed)
    for $\gamma=0.05, 0.15,$ and $0.333$, with $k=10$.} 
    \label{fig:ensemble_comparison}
\end{figure}

For fixed $\alpha>0$, $h(\alpha,k)-k\to0$ as $k\to\infty$.
When the selection threshold $k$ also tends to infinity, the logistic bounds tighten to curves separated by half an index.
We capture this by introducing the ``offset'' parameter $d = k - j$, measuring distance from the selection threshold.

\begin{theorem}[Asymptotic weights for large $k$]
     \label{thm:wd}
     Fix $d\in\mathbb{Z}$ and $\gamma\in(0,1)$, and set $\alpha=-\log(1-\gamma)$. Along any integer sequence with $1\leq k,m\leq p$, the limiting weights $\llweight{d}{\gamma} = \lim\limits_{\substack{k,p\to\infty,\,k/p\to0,\\m/p\to\gamma}}\weight{k-d}{k,m,p}$
     \begin{enumerate}[label=(\alph*)]
     \item \label{thm:wd:a} exist and equal
     \begin{equation}
         \label{eq:closed_wd}
          \sum_{i=0}^{\infty}(-1)^{i}e^{-\alpha i(d+i/2+1/2)},     
     \end{equation}
     \item \label{thm:wd:b} satisfy the symmetry property $\llweight{d}{\gamma} = 1-\llweight{-(d+1)}{\gamma}$,
     \item \label{thm:wd:c} satisfy the bounds
     \begin{enumerate}[label=(\roman*)]
     \label{eq:bounds_wd}
        \item \label{thm:wd:c:i} $ \frac{1}{1+e^{-\alpha d}}  \leq \llweight{d}{\gamma} \leq \frac{1}{1+e^{-\alpha(d+1/2)}}, \quad d < 0 $,
         \item  \label{thm:wd:c:ii} $ \frac{1}{1+e^{-\alpha(d+1/2)}}  \leq \llweight{d}{\gamma} \leq \frac{1}{1+e^{-\alpha(d+1)}}, \quad d \geq 0$,
     \end{enumerate}
     \item \label{thm:wd:d} form a nondecreasing sigmoidal sequence.\footnote{With
     $\Delta^2\llweight{d}{\gamma}=\llweight{d-1}{\gamma}-2\llweight{d}{\gamma}+\llweight{d+1}{\gamma}$,
     the centered second differences are nonnegative at $d\leq-1$ and nonpositive at $d\geq0$.}
     \end{enumerate}
The convergence is uniform in rank, with $\max_{1\leq j\leq p}\big|\weight{j}{k,m,p}-\llweight{k-j}{\gamma}\big|\to0$ along the same sequences.
\end{theorem}

\section{Main results}

We now present our main results comparing EFS($k,m$) to FS($k$) in terms of the training error and degrees of freedom decomposition \eqref{eq:expected_error}.

\subsection{Expected training error}

Since FS($k$) coincides with EFS($k,m$) when $m=p$, it is natural to ask whether a fixed choice of $m<p$ can give smaller expected training error.
We show that such a deterministic choice exists and \emph{quantify} the gap in expected training error.

\begin{theorem}[Training error improvement]
\label{thm:training}
Grant Assumption \ref{ass:orthogonal} and fix $k \geq 1$. Assume $p\geq 2k$, $\sum_{j=k+1}^{2k}\E[\hat\beta_{(j),p}^2]>0$, and $ p $ is large enough that
\[
p \geq 10\frac{\sum_{j=1}^k \mathbb{E}\big[\hat\beta^2_{(j),p}\big]}{\sum_{j=k+1}^{2k}\mathbb{E}\big[\hat\beta^2_{(j),p}\big]}.
\]
Then there exists a deterministic $m_p^\star<p$, independent of the observed response, such that
\begin{equation}
\label{eq:error_gap}
\mathbb{E}\big[\|\mathbf{y}-\hat{\mathbf{f}}_{k,m_p^\star,p}\|^2\big]
\leq \mathbb{E}\big[\|\mathbf{y}-\hat{\mathbf{f}}_{k,p}\|^2\big]
-\frac{1}{4}\frac{\big(\frac{1}{k}\sum_{j=k+1}^{2k}\mathbb{E}\big[\hat\beta^2_{(j),p}\big]\big)^2}{\frac{1}{k}\sum_{j=1}^k \mathbb{E}\bigl[\hat\beta^2_{(j),p}\bigr]}.
\end{equation}
\end{theorem}

Because FS orders variables by observed correlations, the condition and bound use the response induced coefficients $\hat\beta_{(j),p}$.
To interpret the condition, write
\[
A_{k,p} = \frac{1}{k}\sum_{j=1}^k\E\bigl[\hat\beta_{(j),p}^2\bigr], \qquad
B_{k,p} = \frac{1}{k}\sum_{j=k+1}^{2k}\E\bigl[\hat\beta_{(j),p}^2\bigr].
\]
When the ratio $A_{k,p}/B_{k,p}$ is near one, spreading weight to the next $k$ coefficients can improve the fit at little cost to the first $k$ coefficients, whereas a sharp drop is less favorable.
The proof constructs $m_p^\star$ from these averages, independently of the realized response. See \eqref{eq:ea_subst} in Appendix~\ref{subsec:proof_of_training_err}.

\begin{example}[Constant regression function plus Gaussian noise]
\label{ex:constant_regression}
Suppose $y_i=\mu+\varepsilon_i$, with constant $\mu$, IID $\mathcal N(0,\sigma^2)$ noise, and centered features.
For fixed $k$ and $\sigma>0$, the expected squared OLS coefficients near the selection cutoff have comparable magnitudes as $p\to\infty$ with $n=n_p\geq p+1$, so $A_{k,p}/B_{k,p}\to1$.
Thus Theorem~\ref{thm:training} gives smaller expected training error for all sufficiently large $p$. Appendix~\ref{subsec:constant_regression} derives the resulting bound.
\end{example}

Improvement over $m=p$ alone, however, need not imply an interior optimum. Appendix~\ref{subsec:nonmonotone_training} gives a separate example where expected training error has a strict minimum at $1<m<p$.

\subsection{Degrees of freedom}
\label{subsec:dof_main}

We now show that, under a natural assumption, decreasing the feature subsampling ratio $m/p$ does not increase the degrees of freedom.
The key is a decomposition of the degrees of freedom of EFS($k,m$) as a weighted sum of successive increments in the degrees of freedom of FS($j$).

\begin{theorem}[Degrees of freedom improvement]
\label{thm:increasingm}
Grant Assumption \ref{ass:orthogonal}. For $0\leq k\leq p$ and $1\leq m\leq p$, the degrees of freedom satisfy
    \begin{equation} \label{eq:df_alt_decomp}
        \df(\hat f_{k,m,p}) = \sum_{j=1}^p \weight{j}{k,m,p}\paren*{\df(\hat f_{j,p}) - \df(\hat f_{j-1,p})},
    \end{equation}
where $\hat f_{0,p} = 0$ and $\df(\hat f_{0,p}) = 0$.
In particular, if $j\mapsto\df(\hat f_{j,p})$ is concave on $\{0,\ldots,p\}$, then $m\mapsto\df(\hat f_{k,m,p})$ is nondecreasing.
\end{theorem}

\begin{corollary}[Prediction error improvement]
\label{cor:prediction}
Under the conditions of Theorem \ref{thm:training}, suppose also that $j\mapsto\df(\hat f_{j,p})$ is concave.
Then the same deterministic $m_p^\star$ in Theorem \ref{thm:training} satisfies
\[
\df(\hat f_{k,m_p^\star,p})\leq\df(\hat f_{k,p}),\qquad
\E\big[\|\bbf-\hat{\mathbf{f}}_{k,m_p^\star,p}\|^2\big]
<\E\big[\|\bbf-\hat{\mathbf{f}}_{k,p}\|^2\big].
\]
\end{corollary}

The concavity condition in Theorem \ref{thm:increasingm} means that each additional FS step adds less covariance complexity than the previous one.
It is plausible at low signal to noise ratio (SNR), when $ |\hat\beta_{(k),p}| \gg |\beta_{(k),p}|$.
Indeed, in such cases,
\begin{equation}
\label{eq:low_snr_concavity}
\frac{\sigma^2}{n}\paren*{\df(\hat f_{k,p}) - \df(\hat f_{k-1,p})}
= \frac{\sigma^2}{n}\df\big(\hat\beta_{(k),p}x_{(k),p}\big)
= \E\big[\hat\beta_{(k),p}\big(\hat\beta_{(k),p} - \beta_{(k),p}\big)\big]
\approx \E\big[\hat\beta_{(k),p}^2\big],
\end{equation}
where the final expression, $\E[\hat\beta_{(k),p}^2]$, is weakly decreasing in $k$ by Definition~\ref{def:betaj}, suggesting concavity when the approximation is accurate.
Empirical support for the concavity of FS($k$) degrees of freedom, even in settings that are not orthogonal, appears in Figures 1 and 4 of \citet{hastie2020extended} and Figure 7 of \citet{mentch2020randomization}.

\begin{remark}
In the constant regression setting of Example~\ref{ex:constant_regression}, $\beta_{(j),p}=0$ for every rank $j$, so the approximation in \eqref{eq:low_snr_concavity} becomes exact and the degrees of freedom form a concave sequence.
\end{remark}

The following lemma makes this intuition precise by giving a sufficient condition for concavity with a nonzero signal.

\begin{lemma}[Concavity at low SNR]
\label{lem:low_snr_concavity}
Grant Assumption \ref{ass:orthogonal} and suppose the noise is Gaussian with $\sigma>0$.
Let $s=|\{j:\beta_{j,p}\ne0\}|$ be the number of nonzero projected signal coefficients.
For every $p\geq3$ and $1\leq s\leq p$, the sequence $k\mapsto\df(\hat f_{k,p})$ is concave on $\{0,1,\ldots,p\}$ whenever
\begin{equation}
\label{eq:sparse_concavity_bound}
\max_{1\leq j\leq p}|\beta_{j,p}|
\leq\frac{\sigma}{\sqrt n}
\max\left\{1,\sqrt2\big(\sqrt{\log p}-\sqrt{\log s}-\sqrt{\log\log p}\big)\right\}.
\end{equation}
The sequence is also concave when $s=0$ or $p=2$.
\end{lemma}

The same coefficient bound also gives $A_{k,p}/B_{k,p}\to1$ for fixed $\sigma>0$ as $p\to\infty$, with $n=n_p\geq p$ and $k=o(p)$, as shown in Lemma~\ref{lem:weak_signal_second_moment}.
Thus these weak signals satisfy both conditions of Corollary~\ref{cor:prediction}, giving smaller expected training and prediction errors for the same deterministic $m_p^\star$ when $p$ is sufficiently large.

When concavity fails, the degrees of freedom need not be monotone in $m$.
This can happen when a sharp boundary separates signal from noise, as in the sparse setting below, where the coefficient signal to noise ratios grow with $n$.

\begin{proposition}[Nonmonotonic degrees of freedom]
\label{prop:non_monotonic_df}
   Fix an integer $k\geq1$, $\sigma>0$, and $\gamma\in(0,1)$. Grant Assumption~\ref{ass:orthogonal} along any integer sequence with $p\geq k$, $p\to\infty$, $n=n_p\geq p$, and $1\leq m=m_p\leq p$ satisfying $m_p/p\to\gamma$. Suppose $\bbf=\sum_{j=1}^k\bx_{\cdot j,p}$ in the fixed original dictionary and the noise is Gaussian. Then
   \[
    \lim_{p\to\infty} \frac{\df(\hat f_{k,m_p,p}) - \df(\hat f_{k,p})}{\log p}
    =\begin{cases}2(1-\gamma),&k=1,\\+\infty,&k\geq2.\end{cases}
   \]
   For $k\geq2$, more precisely, the gap multiplied by $\sigma/\sqrt n$ has a finite, strictly positive limit depending only on $k,\gamma$.
\end{proposition}

\section{Further comparisons}

\paragraph{Comparison with elastic net.}\label{sec:comp_elastic_net}

Elastic net regression \citep{zou2005regularization} combines the $\ell_1$ and $\ell_2$ penalties of lasso and ridge.
Under Assumption~\ref{ass:orthogonal}, the ``na\"ive elastic net'' before rescaling has coefficients, for $\lambda_1,\lambda_2\geq0$,
\begin{equation}
\label{eq:elastic_net}
\hat{\beta}^{\mathrm{EN}}_{(j),p} = \frac{(|\hat \beta_{(j),p}| - \lambda_1/2)_{+}}{1 + \lambda_2}\operatorname{sgn}(\hat \beta_{(j),p}).
\end{equation}
For comparison, the logistic approximation \eqref{eq:logistic} gives the EFS($k,m$) coefficient, for large $m,p$ with $m/p\approx\gamma$, as
\[
\hat\beta^{\mathrm{EFS}}_{(j),p}
\approx\frac{\hat\beta_{(j),p}}{1+e^{-\alpha(h(\alpha,k)+1/2-j)}}.
\]
From \eqref{eq:elastic_net}, we see that even in the orthogonal setting, the ridge penalty shrinks all coefficients retained by lasso by the common factor $1/(1+\lambda_2)$, without changing their support at fixed $\lambda_1$.
Feature subsampling instead redistributes a fixed total inclusion weight $k$ across ranks, smoothing the hard threshold of FS.
The two families of attainable coefficient profiles therefore differ, although they can coincide for particular data and parameter choices.

The comparisons of BSS, FS, and lasso in \citet{hastie2020extended} likewise show that different solution paths perform best in different SNR regimes. \citet{ghosh2025signal} give related SNR dependent minimax results for parameter risk under random isotropic Gaussian designs.
Under the conditions of Corollary~\ref{cor:prediction}, the EFS path improves expected training error without increasing degrees of freedom relative to BSS.

\paragraph{Revisiting shrinkage.}

Decreasing $m$ does not shrink every coefficient (some weights increase and others decrease), but it does shrink the norm of the fitted values.

\begin{proposition}\label{thm:shrinkage}
Grant Assumption \ref{ass:orthogonal}. For $0\leq k\leq p$ and $1\leq m<m'\leq p$,
\[
\|\hat{\mathbf{f}}_{k,m,p}\| \leq \|\hat{\mathbf{f}}_{k,m',p}\|.
\]
\end{proposition}

The extremes are as follows. Maximal shrinkage occurs at $m=1$, where $\|\hat{\mathbf{f}}_{k,1,p}\|=(k/p)\|\hat{\bbf}_{\operatorname{OLS}}\|$. No additional shrinkage from randomization occurs at $m=p$, where EFS($k,m$) reduces to FS($k$).

\section{Another mechanism for training error reduction}
\label{sec:local_minima}

Theorem \ref{thm:training} establishes training error reduction via the ``enlarged model class'' mechanism of \citet{curth2024random}. The orthogonality assumption makes FS equivalent to BSS, so randomization helps only by searching over a richer family of coefficient profiles.
However, \citet{liu2025randomization} identified a second mechanism. Feature subsampling can help \emph{individual base estimators} avoid poor greedy choices.

We now construct examples demonstrating this second mechanism.
In the orthogonal setting, FS($k$) always selects the $k$ features most correlated with the response, which means that individual EFS base estimators cannot improve its training error.
We therefore introduce correlated features for which random candidate sets can prevent an initially attractive feature from being selected, allowing a different sequence of choices to reach a better fit.
In these examples, FS ends at a strict local minimum. Replacing any one selected feature by any unselected feature increases the least squares training error.

\begin{proposition}[Escaping local minima]
\label{prop:escape_local_min}
For every integer tuple $(n,p,k,m)$ with $n\geq p$, $2\leq k\leq p-2$, and $1\leq m<p$, there exist deterministic normalized, linearly independent feature columns and a noiseless response for which FS($k$) selects a strict local minimum under exchanges of a single feature, but
\[
\|\by-\hat\bbf_{k,m,p}\|^2
\leq\E_{\theta}\big[\|\by-\check\bbf_{k,m,p}\|^2\big]
<\|\by-\hat\bbf_{k,p}\|^2,
\]
where $\check f_{k,m,p}$ is one randomized FS base estimator and $\theta$ denotes randomness in candidate sampling. The design and response may depend on $(n,p,k,m)$.
\end{proposition}

Appendix~\ref{subsec:finite_correlated_guarantee} also gives a finite sample guarantee for a fixed family of correlated designs, including improvement for every $2\leq m<p$ when $\max\{3,p-m+1\}\leq k\leq p-1$.

\section{Simulations}

To assess the generality of Theorems \ref{thm:training} and \ref{thm:increasingm} beyond orthogonal designs, we conducted a simulation study.
We generated IID training data $\{(\bx_i, y_i)\}_{i=1}^n$ with $n = 1000$, $p = 100$, where $\bx_i \sim \mathcal{N}(\boldsymbol{0}, \bSigma)$ with first order autoregressive (AR(1)) correlation $\Sigma_{ij} = \rho^{|i-j|}$ for $\rho = 0.5$, $y_i = \bbeta^\top \bx_i + \varepsilon_i$ with $\varepsilon_i \sim \mathcal{N}(0, \sigma^2)$, and exact sparsity $\beta_j = \mathbf{1}(j \leq 10)$. The noise variance $\sigma^2$ was chosen to achieve the desired signal to noise ratio $\textnormal{SNR} = \bbeta^{\top}\bSigma\bbeta / \sigma^2$.

For EFS, we used $B = 500$ independent base learners and selected $(k,m)$ via ten fold cross validation, then evaluated the full path for $k \in [20]$ at the selected $m$.
Both methods fitted an intercept.
We compared covariance degrees of freedom and training error for EFS($k,m$) against FS($k$) at the same number of FS steps $k$ per base learner, averaging over 1,000 independent datasets at each SNR.
We plot these two quantities directly because \eqref{eq:expected_error} decomposes fixed design prediction error into expected training error and degrees of freedom.
These simulations concern a tuned procedure under correlated designs, whereas the theorems compare estimators at a fixed $m$ under orthogonality.
Appendix~\ref{sec:simulation_details} gives the simulation details and uncertainty estimates for paired differences.
Code and synthetic results are available at \url{https://github.com/jasonklusowski/randomization-greedy-search}.

Figure \ref{fig:df_comp} shows results for $\textnormal{SNR} \in \{0.05, 0.25, 1.0\}$.
At low SNR ($0.05$ and $0.25$), EFS reduces degrees of freedom and prediction error relative to FS at every $k$.
Training error decreases at every $k$ for SNR $0.25$. At SNR $0.05$, it increases for $k\leq3$ and decreases for $k\geq7$, with the differences at $k=4,5,6$ within Monte Carlo uncertainty.
At high SNR ($1.0$), degrees of freedom instead increase near the true sparsity level $k=10$, consistent with Proposition \ref{prop:non_monotonic_df}, despite the violation of Assumption~\ref{ass:orthogonal}.
Prediction error increases at $k=10$ and decreases at every other $k$.

\begin{figure}[ht]
\centering
    \includegraphics[width=\linewidth]{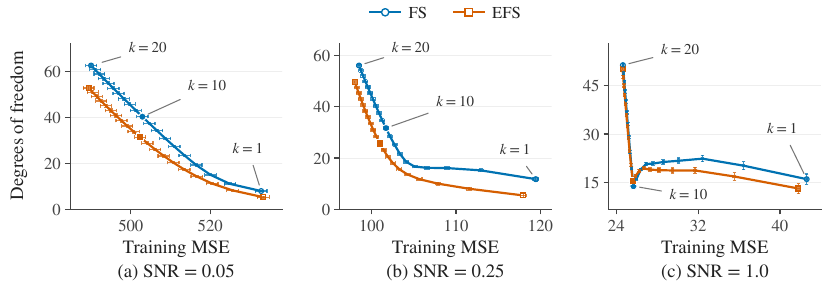}
     \caption{Degrees of freedom vs.\ training MSE for EFS($k,m$) and FS($k$) under a nonorthogonal design with $p/n = 0.1$, AR(1) correlation and exact sparsity, and $k \in [20]$ (true sparsity is 10).
     Means are over 1,000 independent datasets at each SNR, with pointwise 95\% Monte Carlo intervals in each coordinate.
Markers identify $k=1,10,20$, with filled markers at $k=10$.
     Appendix~\ref{sec:simulation_details} reports uncertainty for paired differences.}
    \label{fig:df_comp}
\end{figure}

\section{Discussion}
\label{sec:conclusion}

Under an orthogonal fixed design and the conditions of Theorems \ref{thm:training} and \ref{thm:increasingm}, a fixed EFS($k,m$) estimator can have smaller expected training error and no larger degrees of freedom than FS($k$), and hence smaller prediction error.

\citet{mentch2020randomization} concluded that feature subsampling is mainly useful in low SNR regimes, but our results show it can also help in the absence of noise (see Proposition \ref{prop:escape_local_min}).
Mental models derived from ridge and lasso regression may not transfer cleanly to algorithmic regularization via feature subsampling, dropout, or gradient descent~\citep{wager2013dropout,neyshabur2015searchrealinductivebias,arora2019implicit,bartlett2020benign}.

The main theoretical limitation is orthogonality, partially addressed by simulations with a correlated design.
Our exact results concern EFS; comparable guarantees for full RFs with tree structured split availability and bootstrap sampling remain open.

\section*{Reproducibility}
The code used to reproduce all numerical experiments in this paper is available at \\
\href{https://github.com/jasonklusowski/randomization-greedy-search}{\faGithub\,github.com/jasonklusowski/randomization-greedy-search}.
\begin{acknowledgements}
\label{sec:ai_disclosure}
Klusowski gratefully acknowledges support from the National Science Foundation through NSF CAREER grant DMS-2239448 and from the Alfred P. Sloan Foundation through a Sloan Research Fellowship.
Tan was supported by NUS Startup Grant A-8000448-00-00 and the Singapore Ministry of Education (MOE) AcRF Tier 1 Grants A-8002498-00-00 and A-8004458-00-00.

During revision, we used OpenAI Codex (GPT-6 Astra) to assist with checking mathematical arguments, and developing and testing the replication code. The authors reviewed the resulting material and take responsibility for the final manuscript and code.
\end{acknowledgements}
\bibliography{refs}
\newpage
\setcounter{page}{1}
\renewcommand{\thepage}{\arabic{page}}
\setcounter{section}{0}
\setcounter{equation}{0}
\setcounter{theorem}{0}
\renewcommand{\thesection}{S.\arabic{section}}
\renewcommand{\theequation}{S.\arabic{equation}}
\renewcommand{\thetheorem}{S.\arabic{theorem}}

\section{Additional simulation details}
\label{sec:simulation_details}

\subsection{Correlated designs and Monte Carlo uncertainty}

For the correlated design experiment, EFS varied $m$ over the grid
$\{2,\allowbreak3,\allowbreak5,\allowbreak8,\allowbreak12,\allowbreak19,\allowbreak29,\allowbreak44,\allowbreak66,\allowbreak100\}$.
Ten fold cross validation jointly selected $(k,m)$ over this grid and $k\in[20]$.
We then refitted the complete path at the selected $m$, so $m$ was the same for all displayed values of $k$ within a replicate.
Each EFS fit averaged $B=500$ independent randomized FS paths, with uniform candidate sampling without replacement at each step.
Both methods fitted an unpenalized intercept, and each training fold was centered separately.

We generated 1,000 independent datasets at each SNR and repeated the entire tuning and fitting procedure on each dataset.
FS and EFS used the same data within each replicate for a paired comparison.
For a fitted vector $\hat{\bbf}$, we recorded training error $\|\by-\hat{\bbf}\|^2$, prediction error $\|\bbf-\hat{\bbf}\|^2$, and
\[
\frac{1}{\sigma^2}\sum_{i=1}^n (y_i-f(\bx_i))(\hat f(\bx_i)-f(\bx_i)),
\]
where $\bbf$ is the known mean vector at the training design.
Averaging the last quantity estimates covariance degrees of freedom for the full fitted procedure, including tuning and intercept fitting, averaged over the simulated designs.

For each quantity and $k$, we formed the EFS minus FS difference within each replicate before estimating its mean and Monte Carlo standard error.
The standard error is the sample standard deviation of these paired differences divided by $\sqrt{1000}$.
We obtained approximate pointwise 95\% Monte Carlo intervals using the $t_{999}$ critical value.
These intervals include variation from the data, tuning, and finite ensemble, and do not provide simultaneous coverage across $k$ or SNR.
The coordinate error bars in Figure~\ref{fig:df_comp} use the same calculation separately for each method.

At SNR $0.05$, the paired training error differences at $k=1,2,3$ have pointwise 95\% intervals $[0.329,0.767]$, $[0.429,0.937]$, and $[0.177,0.686]$, respectively.
The paired prediction error intervals instead lie below zero for every $k\in[20]$ at both SNR $0.05$ and $0.25$.
At SNR $1.0$, the corresponding intervals at $k=10,11$ are $[0.0058,0.0325]$ and $[-0.0194,-0.0003]$, respectively.

\clearpage

\section{Proofs}
\label{sec:proofs}

\subsection{Finite ensemble bound}
\label{subsec:finite_ensemble}

For finite $B\geq1$, the fitted vectors satisfy the conditional mean squared error bound
\begin{equation}
\label{eq:finite_ensemble_bound}
\E_\theta\big[\|\hat{\bbf}_{k,m,p,B}-\hat{\bbf}_{k,m,p}\|^2\mid\by\big]
\leq\frac{\|\by\|^2}{B},
\end{equation}
where $\theta$ denotes randomness in candidate sampling with the design fixed.
The bound holds without orthogonality.

\begin{proof}[Proof of \eqref{eq:finite_ensemble_bound}]
Fix the design $\bX$, the response $\by$, and integers $0\leq k\leq p$, $1\leq m\leq p$, and $B\geq1$.
Let $\check{\bbf}_{k,m,p,b}$ be the vector of fitted values from the $b$th randomized FS path, using independent candidate sampling across paths. Then
\[
\hat{\bbf}_{k,m,p,B}=\frac1B\sum_{b=1}^B\check{\bbf}_{k,m,p,b},
\qquad
\hat{\bbf}_{k,m,p}=\E_\theta[\check{\bbf}_{k,m,p,1}\mid\by].
\]
Here, $\theta$ denotes the randomness in candidate sampling, with $\bX$ fixed. All norms use the normalized inner product on response vectors, so $\|\mathbf v\|^2=\mathbf v^\top\mathbf v/n$.
Each path fits by least squares. Thus, if $\activeset_{k,b}$ is its selected set,
\[
\check{\bbf}_{k,m,p,b}=\bP_{\activeset_{k,b}}\by,
\qquad
\|\check{\bbf}_{k,m,p,b}\|^2\leq\|\by\|^2.
\]
The projection onto the selected column span is well defined even when those columns are linearly dependent.

Conditional on $\by$, the centered fitted vectors are independent and have mean zero. Expanding the squared norm of their average, the cross terms therefore have expectation zero, giving
\begin{align*}
\E_\theta\big[\|\hat{\bbf}_{k,m,p,B}-\hat{\bbf}_{k,m,p}\|^2\mid\by\big]
&=\frac1{B^2}\sum_{b=1}^B
\E_\theta\big[\|\check{\bbf}_{k,m,p,b}-\hat{\bbf}_{k,m,p}\|^2\mid\by\big]\\
&=\frac1B\left(\E_\theta\big[\|\check{\bbf}_{k,m,p,1}\|^2\mid\by\big]
-\|\hat{\bbf}_{k,m,p}\|^2\right)\\
&\leq\frac{\|\by\|^2}{B}.
\end{align*}
This proves the bound without requiring orthogonality or full column rank of the design.
\end{proof}

\subsection{Proofs for finite \texorpdfstring{EFS($k,m$)}{EFS(k,m)} weights}
\label{subsec:proofs_finite}

Throughout these proofs, $j$ denotes an observed rank, not an original feature label. Conditional on the response and its ordering permutation $\pi$, write
$\mathcal R_l=\{j:\pi(j)\in\mathcal M_l\}$ and
$\mathcal C_l=\{j:\pi(j)\in\mathcal V_l\}$ for the selected and candidate ranks. Orthogonality implies that the algorithm selects the smallest rank in $\mathcal C_l$.
For recursive calculations we allow any integer $m\geq1$, always sampling $\min\{m,p-l+1\}$ candidates at step $l$ in a population of size $p$. We set $\weight{j}{0,m,p}=0$ and assign weight zero to ranks outside $\{1,\ldots,p\}$, including the empty population.

\begin{lemma}[Monotonicity]
\label{lem:weights_monotonicity}
For integers $p,m\geq1$, $0\leq k\leq p$, and $1\leq j<j'\leq p$, we have $\weight{j}{k,m,p}\geq\weight{j'}{k,m,p}$.
\end{lemma}

\begin{lemma}[Recurrence relation]
\label{prop:finitep}
For integers $p,m\geq1$, $1\leq k\leq p$, and $1\leq j\leq p$, put $\bar m=\min\{m,p\}$. The weights satisfy the recurrence relation
\begin{equation}
\label{eq:r_finite}
\begin{aligned}
\weight{j}{k,m,p}
&=\frac{\binom{p-j}{\bar m-1}}{\binom{p}{\bar m}}
 +\frac{\binom{p-j}{\bar m}}{\binom{p}{\bar m}}\weight{j}{k-1,m,p-1}\\
&\quad+\left(1-\frac{\binom{p-j+1}{\bar m}}{\binom{p}{\bar m}}\right)\weight{j-1}{k-1,m,p-1},
\end{aligned}
\end{equation}
where $\binom ab=0$ for integers $a\geq0$ and $b>a$ or $b<0$. In particular, $\weight{j}{k,m,p}=\mathbf1(j\leq k)$ when $m\geq p$.
\end{lemma}

\begin{proof}[Proof of Lemma \ref{lem:weights_monotonicity}]
To prove monotonicity of the weights with respect to $j$, recall that $\weight{j}{k,m,p}=\mathbb P(j\in\mathcal R_k)$.
Here, the randomness is over the choices of the candidate sets $\mathcal C_1,\ldots,\mathcal C_k$.
It suffices to compare adjacent ranks $a=j-1$ and $b=j$, with $2\leq j\leq p$. Define the events
\[
\mathcal E_b=\{b\in\mathcal R_k,\ a\notin\mathcal R_k\},
\qquad
\mathcal E_a=\{a\in\mathcal R_k,\ b\notin\mathcal R_k\}.
\]
On a path of candidate sets in $\mathcal E_b$, exchange $a$ and $b$ in every candidate set.
Every selected rank other than $b$, on the original path, lies outside $\{a,b\}$. If its candidate set contains $a$ or $b$, it is smaller than both, since they are adjacent and $a$ is never selected.
Thus the exchange preserves every other selection and replaces the selection of $b$ by $a$. The eligible sets are exchanged in the same way, so the transformed path is valid and belongs to $\mathcal E_a$.
The exchange is its own inverse on sequences of candidate sets, hence is injective. Every valid path has the same probability
\[
\prod_{l=1}^k\binom{p-l+1}{\min\{m,p-l+1\}}^{-1}.
\]
Therefore $\mathbb P(\mathcal E_b)\leq\mathbb P(\mathcal E_a)$, which proves the inequality for adjacent ranks and hence the lemma.
\end{proof}

\begin{proof}[Proof of Lemma \ref{prop:finitep}]
Let $J$ denote the rank selected at the first step. The event $J>r$ means that all $\bar m$ candidates exceed $r$, so
\[
\mathbb P(J>r)=\frac{\binom{p-r}{\bar m}}{\binom p{\bar m}},
\qquad r\in\{0,\ldots,p\}.
\]
If $J=j$, rank $j$ is selected. If $J>j$, deleting the first selected rank leaves $j$ at rank $j$ in a population of size $p-1$. If $J<j$, it moves to rank $j-1$. The remaining $k-1$ steps use the same candidate size $m$, using all remaining features if fewer are available. Therefore, by the law of total probability,
\begin{equation}
\label{eq:decomp_sum_int}
\weight{j}{k,m,p}
=\mathbb P(J=j)+\mathbb P(J>j)\weight{j}{k-1,m,p-1}
 +\mathbb P(J<j)\weight{j-1}{k-1,m,p-1}.
\end{equation}
Substituting the survival probabilities above and using
$\mathbb P(J=j)=\mathbb P(J>j-1)-\mathbb P(J>j)$ gives \eqref{eq:r_finite} by the Pascal identity.
The zero conventions make the formula valid also at $j=1$, $j=p$, and $p=1$. When $m\geq p$, every step uses all remaining ranks and selects them in order.
\end{proof}

\subsection{Proofs for asymptotic \texorpdfstring{EFS($k,m$)}{EFS(k,m)} weights}
\label{subsec:proofs_asymptotic}

To prove Theorem \ref{thm:asymp}, we give several useful lemmas. Throughout this subsection, fix $0<\gamma<1$ and write $q=1-\gamma=e^{-\alpha}\in(0,1)$. Unless another limit is specified, $j\geq1$ and $k\geq0$ are fixed integers. The first lemma establishes existence of the limiting weights and gives their recurrence.

\begin{lemma}
\label{lm:infinite_p}
The limit $\limitweight{j}{k,\gamma}=\lim\limits_{\substack{m\to\infty,\,p\to\infty,\\m/p\to\gamma}}\weight{j}{k,m,p}$ exists. For $k\geq1$ it satisfies
\begin{equation}
\label{eq:r_infty}
\limitweight{j}{k,\gamma}
=(1-q)q^{j-1}+q^j\limitweight{j}{k-1,\gamma}
 +(1-q^{j-1})\limitweight{j-1}{k-1,\gamma},
\end{equation}
with $\limitweight{j}{0,\gamma}=\limitweight{0}{k,\gamma}=0$.
\end{lemma}

\begin{proof}[Proof of Lemma \ref{lm:infinite_p}]
To show the limit form of the coefficients in recurrence \eqref{eq:r_finite}, we derive the following result using basic asymptotic facts. For a fixed $j$, for $m,p\to\infty$ and $m/p\to\gamma$, we have
\begin{align*}
\frac{\binom{p-j}{m}}{\binom pm}
&=\prod_{i=0}^{j-1}\frac{p-m-i}{p-i}\longrightarrow q^j,\\
\frac{\binom{p-j}{m-1}}{\binom pm}
&=\frac mp\prod_{i=0}^{j-2}\frac{p-m-i}{p-1-i}
 \longrightarrow(1-q)q^{j-1}.
\end{align*}
The empty product for $j=1$ equals one. The third coefficient tends to $1-q^{j-1}$ by the first formula with $j-1$ in place of $j$.

Now, we show that $\limitweight{j}{k,\gamma}$ is well defined via induction on $k$, simultaneously for all fixed ranks. For the base case, note that $\weight{j}{0,m,p}=0$ for all $j$ and all $p$, so $\limitweight{j}{0,\gamma}=0$ is well defined for all $j$.
Now suppose that $\limitweight{j}{k,\gamma}$ is well defined for all steps up to $k$ and for all $j$. Consider step $(k+1)$. In recurrence \eqref{eq:r_finite}, the two weights on the right hand side are $\weight{j}{k,m,p-1}$ and $\weight{j-1}{k,m,p-1}$. Since $m/(p-1)\to\gamma$, they converge by the induction hypothesis. At $j=1$ the term at rank zero is zero by convention. Also, $m<p-k$ eventually, so the sampling cap causes no difficulty at any of these states with a fixed number of steps. Therefore, taking limits in \eqref{eq:r_finite}, we obtain the limit at step $k+1$ and the recurrence in Lemma \ref{lm:infinite_p}, completing the induction.
\end{proof}

\begin{lemma}[Monotonicity of limiting weights]
\label{lm:increasingk}
For all $j\geq1$ and $k\geq0$,
\[
\limitweight{j}{k,\gamma}\leq\limitweight{j}{k+1,\gamma},
\qquad
\limitweight{j}{k,\gamma}\geq\limitweight{j+1}{k,\gamma}.
\]
\end{lemma}
\begin{proof}[Proof of Lemma \ref{lm:increasingk}]
For each finite population, couple the fits at steps $k$ and $k+1$ along the same randomized path. Its selected rank sets are nested, so the inclusion probabilities are nondecreasing in $k$. Their rank monotonicity follows from Lemma~\ref{lem:weights_monotonicity}. Taking the limits at fixed ranks and steps gives both inequalities.
\end{proof}

However, recurrence \eqref{eq:r_infty} contains dependencies on both $\limitweight{j}{k-1,\gamma}$ and $\limitweight{j-1}{k-1,\gamma}$. These dependencies make it difficult to construct a closed form and uncover properties of $\limitweight{j}{k,\gamma}$. Ideally, we would like recurrences involving a change in $j$ to be separate from recurrences involving a change in $k$. Fortunately, recurrence \eqref{eq:r_infty} can be separated into the following pair of recurrences.
\begin{lemma}
\label{lm:recursion}
From the recurrence in Lemma \ref{lm:infinite_p}, we can construct the following recurrence pair for $j,k\geq1$
\begin{equation}
\label{eq:r1}
\limitweight{j}{k,\gamma}
=(q^{j-k}-q^j)(1-\limitweight{j}{k-1,\gamma}),
\end{equation}
\begin{equation}
\label{eq:r2}
\limitweight{j}{k,\gamma}
=1-q^k-(q^{k-j+1}-q^k)\limitweight{j-1}{k,\gamma}.
\end{equation}
\end{lemma}
\begin{proof}[Proof of Lemma \ref{lm:recursion}]
For this proof, write $u_j^k=\limitweight{j}{k,\gamma}$. At rank one, \eqref{eq:r_infty} gives $u_1^k=(1-q)+qu_1^{k-1}=1-q^k$, which satisfies both recurrences directly. We will induct on $k$, simultaneously for all $j\geq2$. First, we verify that both recurrences satisfy the base case of $k=1$. Substituting $u_j^1=(1-q)q^{j-1}$ and $u_j^0=0$ verifies both formulas.

Now, suppose that both recurrences are valid for steps up to $k$ and all $j$. Consider step $(k+1)$. Multiplying \eqref{eq:r2} by $q^{j-1-k}$ gives
\[
(1-q^{j-1})u_{j-1}^k
=q^{j-1-k}(1-q^k-u_j^k).
\]
Substituting this identity into the recurrence in Lemma \ref{lm:infinite_p}, we obtain the following result
\[
u_j^{k+1}=(q^{j-k-1}-q^j)(1-u_j^k),
\]
which proves \eqref{eq:r1} at step $k+1$ for every rank. Now, we prove that recurrence \eqref{eq:r2} is satisfied for step $(k+1)$ in a similar fashion. Put $A_j=q^{j-k-1}(1-q^{k+1})>0$. We plug in recurrence \eqref{eq:r1} at ranks $j$ and $j-1$ into Lemma \ref{lm:infinite_p}, substituting $u_j^k=1-u_j^{k+1}/A_j$ and $u_{j-1}^k=1-u_{j-1}^{k+1}/A_{j-1}$. Since
\[
\frac{q^j}{A_j}=\frac{q^{k+1}}{1-q^{k+1}},
\qquad
\frac{1-q^{j-1}}{A_{j-1}}
=\frac{q^{k-j+2}-q^{k+1}}{1-q^{k+1}},
\]
rearranging terms, we obtain the following recurrence
\[
u_j^{k+1}=1-q^{k+1}-(q^{k-j+2}-q^{k+1})u_{j-1}^{k+1}.
\]
Thus, recurrence \eqref{eq:r2} is satisfied for step $(k+1)$. Therefore, by induction, both recurrences are satisfied for all $j,k\geq1$.
\end{proof}

\begin{proof}[Proof of Theorem \ref{thm:asymp}]
To prove \ref{thm:asymp:a}, note that from Lemma \ref{lm:infinite_p}, we have that $\limitweight{j}{k,\gamma}$ exists. From \eqref{eq:r1}, a first order inhomogeneous recurrence relation with variable coefficients, starting from the weight at step zero, we directly obtain the following closed form for $\limitweight{j}{k,\gamma}$ \citep[Equation 1.2.4]{elaydi2005introduction}
\[
\limitweight{j}{k,\gamma}
=\sum_{i=1}^k(-1)^{k-i}\prod_{l=i}^k\big(q^{j-l}-q^j\big),
\]
which is \eqref{eq:closed_wj}.

To prove \ref{thm:asymp:b}, we first show the lower bound. According to Lemma \ref{lm:increasingk}, we have $\limitweight{j}{k,\gamma}\geq\limitweight{j}{k-1,\gamma}$. From \eqref{eq:r1}, we know
\[
\limitweight{j}{k,\gamma}
\geq(q^{j-k}-q^j)(1-\limitweight{j}{k,\gamma}).
\]
Therefore, solving for $\limitweight{j}{k,\gamma}$, we have
\[
\limitweight{j}{k,\gamma}\geq\frac{1-q^k}{1-q^k+q^{k-j}}.
\]
For the upper bound, take $j\geq2$. The multiplier $q^{k-j+1}-q^k$ is nonnegative, and rank monotonicity in Lemma~\ref{lm:increasingk}, together with \eqref{eq:r2}, gives
\[
\limitweight{j}{k,\gamma}
\leq1-q^k-(q^{k-j+1}-q^k)\limitweight{j}{k,\gamma}.
\]
Solving for $\limitweight{j}{k,\gamma}$, we have
\[
\limitweight{j}{k,\gamma}\leq\frac{1-q^k}{1-q^k+q^{k-j+1}}.
\]
For $j=1$ this upper bound is equality, since $\limitweight{1}{k,\gamma}=1-q^k$. Finally, note that we can write
\[
\frac{1-q^k}{1-q^k+q^{k-j}}
=\frac{1}{1+e^{\alpha j}/(e^{\alpha k}-1)}
=\frac{1}{1+e^{-\alpha(h(\alpha,k)-j)}},
\]
and the upper bound has the same form with $h(\alpha,k)+1$ in place of $h(\alpha,k)$.
\end{proof}

\begin{remark}
\label{rem:weights_properties}
The limiting weights are weakly decreasing in rank by Lemma~\ref{lm:increasingk}, and retain total mass $k$
\[
\sum_{j=1}^{\infty}\limitweight{j}{k,\gamma}=k.
\]
To prove the second property, note that the upper bound in \eqref{eq:bounds_wj} gives a summable geometric bound for each fixed $k$. We may therefore sum and reindex \eqref{eq:r_infty} to obtain
\begin{align*}
\sum_{j=1}^{\infty}\limitweight{j}{k,\gamma}
&=\sum_{j=1}^{\infty}(1-q)q^{j-1}
 +\sum_{j=1}^{\infty}q^j\limitweight{j}{k-1,\gamma}
 +\sum_{j=1}^{\infty}(1-q^{j-1})\limitweight{j-1}{k-1,\gamma}\\
&=1+\sum_{j=1}^{\infty}\limitweight{j}{k-1,\gamma}=k,
\end{align*}
by induction from $k=0$. The restriction $0<\gamma<1$ matters here. When $m/p\to0$, every limiting weight at a fixed rank and step is zero and mass escapes to increasing ranks. When $m/p\to1$, however, the limits at fixed ranks are the weights from hard thresholding $\mathbf1(j\leq k)$.
\end{remark}

We now turn to the proof of Theorem \ref{thm:wd}. Before that, we present a useful lemma regarding an alternative form of $\limitweight{j}{k,\gamma}$.
\begin{lemma}
\label{lm:closed_alt}
For $j,k\geq1$, the asymptotic weights $\limitweight{j}{k,\gamma}$ can be rewritten in the following closed form
\[
\limitweight{j}{k,\gamma}
=\sum_{i=1}^k(-1)^{i+1}e^{\alpha i(k-j-i/2+1/2)}
 \prod_{l=k-i+1}^k(1-e^{-\alpha l}).
\]
\end{lemma}
\begin{proof}[Proof of Lemma \ref{lm:closed_alt}]
Using the closed form from \ref{thm:asymp:a}, we reindex each product by its length $i$. Factoring $e^{-\alpha(j-l)}$ out of its $l$th factor gives
\begin{align*}
\prod_{l=k-i+1}^k\big(e^{-\alpha(j-l)}-e^{-\alpha j}\big)
&=e^{\alpha\sum_{l=k-i+1}^k(l-j)}\prod_{l=k-i+1}^k(1-e^{-\alpha l})\\
&=e^{\alpha i(k-j-i/2+1/2)}\prod_{l=k-i+1}^k(1-e^{-\alpha l}),
\end{align*}
and the sign becomes $(-1)^{i+1}$.
\end{proof}

For the following proofs, we perform the substitution $d=k-j$. We first prove the existence and derive a closed form for $\llweight{d}{\gamma}$ slightly different from \eqref{eq:closed_wd}.
\begin{lemma}
\label{lem:alternative_wd}
For every integer $d$, the limit $\llweight{d}{\gamma}=\lim\limits_{k\to\infty}\limitweight{k-d}{k,\gamma}$ exists and equals
\begin{equation}
\label{eq:closed_wd_naive}
\llweight{d}{\gamma}
=\sum_{i=1}^{\infty}(-1)^{i+1}e^{\alpha i(d-i/2+1/2)}.
\end{equation}
\end{lemma}
\begin{proof}[Proof of Lemma \ref{lem:alternative_wd}]
Fix $d$, let $j=k-d$, and take $k\geq\max\{1,d+1\}$ in Lemma~\ref{lm:closed_alt}. Extend its sum by setting the summands to zero for $i>k$. For each fixed $i$, the product tends to one as $k\to\infty$. The absolute value of the $i$th summand is bounded by
\[
e^{\alpha i(d+1/2)-\alpha i^2/2},
\]
which is summable over $i\geq1$ for fixed $d$ and $\alpha>0$. Therefore, by dominated convergence, the limit $\llweight{d}{\gamma}$ exists and is given by the absolutely convergent series \eqref{eq:closed_wd_naive}.
\end{proof}

\begin{proof}[Proof of Theorem \ref{thm:wd}]
We first prove the closed form and the stated properties for the iterated limit in Lemma~\ref{lem:alternative_wd}. The joint convergence is proved below. To obtain the closed form, we show that \eqref{eq:closed_wd_naive} is the same as \eqref{eq:closed_wd}. Using the difference of squares formula,
\[
\llweight{d}{\gamma}
=e^{\alpha(d+1/2)^2/2}
 \sum_{i=1}^{\infty}(-1)^{i+1}e^{-\alpha(i-d-1/2)^2/2}.
\]
Since $d$ is an integer, pairing $i$ with $2d+1-i$ cancels the terms with $1\leq i\leq2d$ when $d\geq0$. For $d<0$, the same pairing shows that the finite sum over $2d+1\leq i\leq0$ vanishes, so those terms may be inserted. In either case the series may start at $i=2d+1$. Reindexing $i=r+2d+1$ gives
\[
\llweight{d}{\gamma}
=\sum_{r=0}^{\infty}(-1)^re^{-\alpha r(d+r/2+1/2)},
\]
which proves \eqref{eq:closed_wd}. To prove \ref{thm:wd:b}, using the closed form from \eqref{eq:closed_wd_naive}, note that
\[
1-\llweight{-(d+1)}{\gamma}
=1+\sum_{i=1}^{\infty}(-1)^ie^{-\alpha i(d+i/2+1/2)}
=\llweight{d}{\gamma},
\]
proving \ref{thm:wd:b}.

To prove \ref{thm:wd:c}, by taking $k\to\infty$ with $j=k-d$ in the bounds from \eqref{eq:bounds_wj}, we have
\begin{equation}
\label{eq:bounds_wd1}
\frac{1}{1+e^{-\alpha d}}
\leq\llweight{d}{\gamma}
\leq\frac{1}{1+e^{-\alpha(d+1)}}.
\end{equation}
Additionally, using the bounds from \eqref{eq:bounds_wd1}, we have
\begin{equation}
\label{eq:wd_non-dec}
\llweight{d}{\gamma}
\leq\frac{1}{1+e^{-\alpha(d+1)}}
\leq\llweight{d+1}{\gamma},
\end{equation}
demonstrating that $\llweight{d}{\gamma}$ is nondecreasing for $d\in\Z$.

To complete the proof of \ref{thm:wd:c} and prove \ref{thm:wd:d}, define the auxiliary real function
\[
F(t)=\sum_{i=0}^{\infty}(-1)^ie^{-\alpha i(t+i/2+1/2)},
\qquad t\geq-1/2.
\]
It agrees with $\llweight{d}{\gamma}$ at every integer $d\geq0$. Its series and its derivatives converge locally uniformly by Gaussian decay in $i$. In particular, $F$ is continuous at $-1/2$. Now, consider the following identity from \citet[Proposition 1.1]{coogan2003q}, written for $0<z,r<1$
\begin{equation}
\label{eq:sigmoid_identity}
A(z,r)\coloneqq \sum_{i=0}^{\infty}(-1)^iz^{2i}r^{i^2}
=\frac{1}{1+z}\sum_{i=0}^{\infty}z^i
 \prod_{l=0}^{i-1}\frac{1-zr^l}{1+zr^{l+1}}.
\end{equation}
The formal identity applies analytically here. For fixed $r\in(0,1)$, both sides converge locally uniformly in $|z|<1$, since the finite products are uniformly bounded on compact subsets of that disk. They therefore represent the same analytic function.

For $0<z,r<1$, each product factor lies between $(1-z)/(1+z)$ and one. Thus
\begin{equation}
\label{eq:w_lower}
A(z,r)\geq\frac{1}{1+z}
 \sum_{i=0}^{\infty}\left(\frac{z(1-z)}{1+z}\right)^i
=\frac{1}{1+z^2}.
\end{equation}
Set $r=e^{-\alpha/2}$ and $z=e^{-\alpha(t+1/2)/2}$ for $t>-1/2$. Then $F(t)=A(z,r)$ and
\[
F(t)\geq\frac{1}{1+e^{-\alpha(t+1/2)}}.
\]
For integer $d\geq0$, this is the improved lower bound in \ref{thm:wd:c:ii}. Reflection using \ref{thm:wd:b} gives the improved upper bound for $d<0$. Together with \eqref{eq:bounds_wd1}, this proves \ref{thm:wd:c}. Taking $t\downarrow-1/2$ in the last inequality gives $F(-1/2)\geq1/2$. Here we use continuity of the theta series at $t=-1/2$.

For fixed $z\in(0,1)$, each factor $(1-zr^l)/(1+zr^{l+1})$ is nonnegative and nonincreasing in $r\in(0,1)$. Hence every summand on the right of \eqref{eq:sigmoid_identity}, and therefore their sum, is nonincreasing in $r$. The theta series on the left is differentiable in $r$ by Gaussian decay, so $\partial_r A(z,r)\leq0$. At the values of $z,r$ above,
\[
F''(t)=\alpha^2\sum_{i=0}^{\infty}(-1)^ii^2z^{2i}r^{i^2}
=\alpha^2r\,\partial_r A(z,r)\leq0.
\]
Thus $F$ is concave on $[-1/2,\infty)$, by continuity at the endpoint. It follows that the centered second difference of $\llweight{d}{\gamma}$ is nonpositive for every integer $d\geq1$. At $d=0$, concavity between $-1/2$ and $1$ gives
\[
\llweight{0}{\gamma}=F(0)
\geq\frac23F(-1/2)+\frac13F(1)
\geq\frac13+\frac13\llweight{1}{\gamma}.
\]
Since $\llweight{-1}{\gamma}=1-\llweight{0}{\gamma}$, this is exactly
$\llweight{-1}{\gamma}-2\llweight{0}{\gamma}+\llweight{1}{\gamma}\leq0$.
The symmetry in \ref{thm:wd:b} reverses the sign of the centered second difference under $d\mapsto-d-1$, proving nonnegative second differences for $d\leq-1$. Together with \eqref{eq:wd_non-dec}, this proves \ref{thm:wd:d}.
\end{proof}

To pass from the iterated limit to the joint limit in Theorem~\ref{thm:wd}, we next control the finite population error, the dependence on the subsampling ratio, and the finite step error.

\begin{lemma}[Finite population approximation]
\label{lem:finite_population_weights}
For integers $p\geq2$, $1\leq m<p$, and $1\leq k\leq p$, put
$\gamma=m/p$ and $q=1-\gamma$. Then
\begin{equation}
\label{eq:finite_population_weights}
\max_{1\leq j\leq p}
\big|\weight{j}{k,m,p}-\limitweight{j}{k,\gamma}\big|
\leq\min\left\{1,\frac2p\left(\frac{k-1}{\gamma^3}
+\frac{q}{\gamma^4}\right)\right\}.
\end{equation}
\end{lemma}

\begin{proof}[Proof of Lemma \ref{lem:finite_population_weights}]
Consider an infinite list of ranks for a probabilistic representation of the limiting weights. At each step, remove the $G$th
remaining rank, where the successive choices are independent and
$\mathbb P(G=r)=\gamma q^{r-1}$ for $r\geq1$.
Conditioning on the first choice gives recurrence \eqref{eq:r_infty},
so its inclusion probabilities after $k$ steps are
$\limitweight{j}{k,\gamma}$.

Track the current eligible rank $r$ of one fixed original feature until
it is selected. At finite time $t$, let $N=p-t$, let $J_N$ be the
minimum rank in a uniform subset of size $\min\{m,N\}$, and put
$F_t(r)=\mathbb P(J_N>r)$, with $F_t(0)=1$.
We claim that, for $1\leq r\leq N$,
\begin{equation}
\label{eq:finite_survival_error}
0\leq q^r-F_t(r)
\leq\frac{q^{r-1}}p\left(rt+\frac{r(r-1)}2\right).
\end{equation}
For $0\leq i<N$, put $a_i\coloneqq\max\{0,1-m/(N-i)\}$.
Then $F_t(r)=\prod_{i=0}^{r-1}a_i$ for every $1\leq r\leq N$, including when the candidate size is capped by $N$.
Each factor lies in $[0,q]$. If $N-i\geq m$, then
\[
q-a_i=\frac{m(t+i)}{p(N-i)}\leq\frac{t+i}p.
\]
If $N-i<m$, then $a_i=0$ and $t+i>p-m=pq$, so the same upper bound holds.
Expanding the difference of the two products gives
\[
0\leq q^r-F_t(r)
\leq q^{r-1}\sum_{i=0}^{r-1}(q-a_i)
\leq\frac{q^{r-1}}p\left(rt+\frac{r(r-1)}2\right),
\]
which proves \eqref{eq:finite_survival_error} for every eligible rank.

From state $r$, the probabilities of selection, a decrease to rank
$r-1$, and remaining at rank $r$ are respectively
\[
F_t(r-1)-F_t(r),\qquad 1-F_t(r-1),\qquad F_t(r).
\]
The infinite process replaces $F_t(s)$ by $q^s$.
Writing $\Delta_t(s)=q^s-F_t(s)$, the total variation distance between
these distributions on three points is
$\max\{\Delta_t(r-1),\Delta_t(r)\}$.
For $0\leq t\leq k-1$, this is bounded by $b_{r-1}+b_r$, where
$b_0=0$ and
\[
b_r=\frac{q^{r-1}}p\left(r(k-1)+\frac{r(r-1)}2\right),
\qquad r\geq1.
\]

Until their first disagreement, couple the two rank processes so that
their conditional probability of disagreement is this total variation distance.
Continue the infinite process after any disagreement.
In the infinite process, a rank cannot be revisited after it is left,
and the probability of remaining at rank $r$ is $q^r$.
Its expected number of visits to that rank is therefore at most
$(1-q^r)^{-1}\leq\gamma^{-1}$, whether or not the rank is reached.
The probability of a first disagreement is bounded by the expected
accumulated errors along the infinite process. Hence it is at most
\begin{align*}
\frac1\gamma\sum_{r=1}^{\infty}(b_{r-1}+b_r)
&=\frac2{p\gamma}\sum_{r=1}^{\infty}q^{r-1}
\left(r(k-1)+\frac{r(r-1)}2\right)\\
&=\frac2p\left(\frac{k-1}{\gamma^3}+\frac q{\gamma^4}\right).
\end{align*}
If no disagreement occurs, the two inclusion indicators agree.
The bound is independent of its original rank $j$, and the difference
between two inclusion probabilities is also at most one.
\end{proof}

\begin{lemma}[Dependence on the subsampling ratio]
\label{lem:weight_parameter_continuity}
For $\gamma,\gamma'\in(0,1)$, put
$\gamma_0=\min\{\gamma,\gamma'\}$. For every integer $k\geq0$,
\begin{equation}
\label{eq:weight_parameter_continuity}
\sup_{j\geq1}\big|\limitweight{j}{k,\gamma}
-\limitweight{j}{k,\gamma'}\big|
\leq\min\left\{1,\frac{2|\gamma-\gamma'|}{\gamma_0^3}\right\}.
\end{equation}
The same bound holds for
$\sup_{d\in\mathbb Z}\big|\llweight{d}{\gamma}-\llweight{d}{\gamma'}\big|$.
\end{lemma}

\begin{proof}[Proof of Lemma \ref{lem:weight_parameter_continuity}]
The two geometric survival probabilities at rank $r$ differ by at most
\[
r(1-\gamma_0)^{r-1}|\gamma-\gamma'|.
\]
As in the proof of Lemma~\ref{lem:finite_population_weights}, the total
variation distance between the transition distributions is at most
the sum of the two adjacent survival differences, and the expected
number of visits to any rank is at most $1/\gamma_0$.
Summing these errors gives
\[
\frac{2|\gamma-\gamma'|}{\gamma_0}
\sum_{r=1}^{\infty}r(1-\gamma_0)^{r-1}
=\frac{2|\gamma-\gamma'|}{\gamma_0^3}.
\]
The claim is immediate when $k=0$.
For each fixed integer $d$, let $k\to\infty$ with $j=k-d$ and use
Lemma~\ref{lem:alternative_wd}. The bound does not depend on $d$,
which proves the assertion for the offset profile.
\end{proof}

\begin{lemma}[Finite step approximation]
\label{lem:finite_step_weights}
For every integer $k\geq1$ and $\gamma\in(0,1)$, put $q=1-\gamma$.
Then
\begin{equation}
\label{eq:finite_step_weights}
\sup_{j\geq1}\big|\limitweight{j}{k,\gamma}
-\llweight{k-j}{\gamma}\big|
\leq\min\left\{1,q^k(1+\gamma^{-2})\right\}.
\end{equation}
\end{lemma}

\begin{proof}[Proof of Lemma \ref{lem:finite_step_weights}]
For this proof, write $u_j^t=\limitweight{j}{t,\gamma}$ and $W_d=\llweight{d}{\gamma}$, and put
$e_j^t\coloneqq u_j^t-W_{t-j}$.
Reindexing the absolutely convergent series in \eqref{eq:closed_wd} gives
\[
W_d=1-q^{d+1}W_{d+1}.
\]
First take $j>k$. By \eqref{eq:r1} and the profile identity,
\[
e_j^t=-q^{j-t}e_j^{t-1}-q^j(1-u_j^{t-1}),
\qquad 1\leq t\leq k.
\]
The alternating series in \eqref{eq:closed_wd_naive} has decreasing positive terms starting at $q^j$ when $d=-j$, so $|e_j^0|=W_{-j}\leq q^j$.
Taking absolute values and iterating, using $0\leq u_j^t\leq1$ and $q^{j-t}\leq q^{j-k}$, gives
\[
|e_j^k|\leq q^j\sum_{i=0}^k q^{i(j-k)}
\leq\frac{q^j}{1-q^{j-k}}
\leq\frac{q^k}{\gamma}.
\]
Now take $1\leq j\leq k$. At rank one, $u_1^k=1-q^k$ and the profile identity give
$e_1^k=q^k(W_k-1)$, so $|e_1^k|\leq q^k$.
For $j\geq2$, \eqref{eq:r2} gives
\[
e_j^k=-q^{k-j+1}e_{j-1}^k-q^k(1-u_{j-1}^k).
\]
Taking absolute values and iterating over ranks up to $j$, whose coefficients are at most $q^{k-j+1}$, gives
\[
|e_j^k|\leq q^k\sum_{i=0}^{j-1}q^{i(k-j+1)}
\leq\frac{q^k}{1-q^{k-j+1}}
\leq\frac{q^k}{\gamma}.
\]
Since both weights lie in $[0,1]$ and $\gamma^{-1}\leq1+\gamma^{-2}$, the two cases prove \eqref{eq:finite_step_weights}.
\end{proof}

\begin{proof}[Joint convergence in Theorem \ref{thm:wd}]
Fix $\gamma\in(0,1)$ and let $a=m/p\in(0,1)$.
Lemmas~\ref{lem:finite_population_weights},
\ref{lem:weight_parameter_continuity}, and~\ref{lem:finite_step_weights}
give
\begin{align}
\label{eq:joint_weight_error}
\max_{1\leq j\leq p}
\big|\weight{j}{k,m,p}-\llweight{k-j}{\gamma}\big|
&\leq\frac2p\left(\frac{k-1}{a^3}+\frac{1-a}{a^4}\right)\notag\\
&\quad+(1-a)^k(1+a^{-2})
+\frac{2|a-\gamma|}{\min\{a,\gamma\}^3}.
\end{align}
Along the stated sequences, $a\to\gamma$, $k/p\to0$, and $k\to\infty$,
so every term on the right tends to zero.
This proves uniform convergence over the ranks $1\leq j\leq p$.
For any fixed integer $d$, the rank $j=k-d$ belongs to this range
eventually. Its limit is therefore the profile already computed above,
which proves the assertion about the joint limit and completes the theorem.
\end{proof}

\subsection{Proof of Theorem \ref{thm:training}}
\label{subsec:proof_of_training_err}

\begin{proof}
From the decompositions \eqref{eq:decomp_ensemble} and \eqref{eq:decomp_fs}, we have the following identity
\begin{align}
\label{eq:training_decomp}
    \|\mathbf{y}-\hat{\mathbf{f}}_{k,p}\|^2 - \|\mathbf{y}-\hat{\mathbf{f}}_{k,m,p}\|^2
    &= \sum_{j=k+1}^{p}\hat\beta^2_{(j),p}\Bigl[1-\bigl(1-\weight{j}{k,m,p}\bigr)^2\Bigr] \notag\\
    &\quad - \sum_{j=1}^{k}\hat\beta^2_{(j),p}\big(1-\weight{j}{k,m,p}\big)^2.
\end{align}

Let $m_p$ be deterministic and chosen below. For $p\geq 2k$, taking expectations in \eqref{eq:training_decomp} and discarding the nonnegative terms with $j>2k$ gives
\begin{align}
\label{eq:training_bound}
&\mathbb{E}\Big[\|\mathbf{y}-\hat{\mathbf{f}}_{k,p}\|^2-\|\mathbf{y}-\hat{\mathbf{f}}_{k,m_p,p}\|^2\Big] \notag\\
&\qquad\geq \sum_{j=k+1}^{2k}\mathbb{E}\bigl[\hat\beta^2_{(j),p}\bigr]\Bigl[1-\bigl(1-\weight{j}{k,m_p,p}\bigr)^2\Bigr]
-\sum_{j=1}^{k}\mathbb{E}\bigl[\hat\beta^2_{(j),p}\bigr]\big(1-\weight{j}{k,m_p,p}\big)^2.
\end{align}
To simplify further, set
\begin{align}
\label{eq:limsup_p_notation}
    A_{k,p} = \frac{1}{k}\sum_{j=1}^k\mathbb{E}\bigl[\hat\beta^2_{(j),p}\bigr], \qquad B_{k,p} = \frac{1}{k}\sum_{j=k+1}^{2k}\mathbb{E}\bigl[\hat\beta^2_{(j),p}\bigr],
\end{align}
and note that $A_{k,p}\geq B_{k,p}>0$. By Definition \ref{def:betaj} and Lemma \ref{lem:weights_monotonicity}, both $\mathbb{E}[\hat\beta^2_{(j),p}]$ and $\weight{j}{k,m_p,p}$ are weakly decreasing in $j$. The Chebyshev sum inequality \citep{hardy1988inequalities} therefore gives
\begin{align}
\label{eq:training_upper}
    \sum_{j=1}^k\mathbb{E}\bigl[\hat\beta^2_{(j),p}\bigr]\big(1-\weight{j}{k,m_p,p}\big)^2
    &\leq A_{k,p}\sum_{j=1}^k\big(1-\weight{j}{k,m_p,p}\big)^2,
\end{align}
and
\begin{align}
\label{eq:training_lower}
    \sum_{j=k+1}^{2k}\mathbb{E}\bigl[\hat\beta^2_{(j),p}\bigr]\Bigl[1-\bigl(1-\weight{j}{k,m_p,p}\bigr)^2\Bigr]
    &\geq B_{k,p}\sum_{j=k+1}^{2k}\Bigl[1-\bigl(1-\weight{j}{k,m_p,p}\bigr)^2\Bigr].
\end{align}

We first bound the finite weights for $2\leq m_p\leq p-k$. Put
\[
N=p-k+1,\qquad b=N-m_p,\qquad q=\frac bN,\qquad
w_j=\weight{j}{k,m_p,p},
\]
and define
\[
d_j=1-w_j\quad (1\leq j\leq k),\qquad
D=\sum_{j=1}^k d_j=\sum_{j=k+1}^p w_j,\qquad
L=\sum_{j=2k+1}^p w_j.
\]
The equality defining $D$ follows because the weights sum to $k$. Thus $D$ is the total weight assigned beyond the first $k$ ranks, and $L$ is the part assigned beyond rank $2k$.
For $h\geq1$, let
\[
a_h=\prod_{r=0}^{h-1}\frac{b+r}{N+r},\qquad a_0=1.
\]
We claim that
\begin{equation}
\label{eq:finite_training_weights}
d_j\leq a_{k-j+1}\leq q,\qquad
q\leq D\leq\frac{b}{m_p-1},\qquad
L\leq q^k\frac{b}{m_p-1}.
\end{equation}

For the first bound, induct on the number of selections. Let $J$ be the
rank selected at the first step. Conditional on $J<j$, the remaining
problem has population $p-1$, $k-1$ selections and target rank $j-1$.
Conditional on $J>j$, the target rank remains $j$. Both problems have
the same $N$ and $b$. With $h=k-j+1$, induction therefore gives
\[
d_j\leq\mathbb P(J<j)a_h+\mathbb P(J>j)a_{h-1}.
\]
When $j=k$, the second bound uses only $a_0=1$. When $j=1$, the first
event has probability zero. Conditional on $J\geq j$, all candidates
lie among $p-j+1$ ranks, so
\[
\frac{\mathbb P(J>j)}{\mathbb P(J\geq j)}
=\frac{p-j+1-m_p}{p-j+1}
=\frac{b+h-1}{N+h-1}
=\frac{a_h}{a_{h-1}}.
\]
Substitution proves $d_j\leq a_h$. The case $k=1$ initializes the
induction exactly. Every factor in $a_h$ is at most one, so $a_h\leq q$.

After $s$ selections, at least $k-s$ ranks among the first $k$ remain, which we use to bound $D$.
Conditional on any history, the probability that the next selection
has rank greater than $k$ is therefore at most
\[
\frac{\binom{p-k}{m_p}}{\binom{p-s}{m_p}}=a_{k-s}.
\]
Summing over the $k$ steps gives $D\leq\sum_{h=1}^k a_h$.
The recurrence for $a_h$ gives
${(b+h-1)a_{h-1}}-{(b+h)a_h}=(m_p-1)a_h$.
Telescoping therefore yields
\[
\sum_{h=1}^k a_h
=\frac{b-(b+k)a_k}{m_p-1}
\leq\frac{b}{m_p-1}.
\]

If no rank greater than $k$ has been selected before the final step,
exactly one of the first $k$ remains, and it is missed with probability
$q$. Thus the probability of selecting at least one rank greater than
$k$ is at least $q$, which proves $D\geq q$.

The same argument for the first $2k$ ranks gives
\[
L\leq\sum_{s=0}^{k-1}
\frac{\binom{p-2k}{m_p}}{\binom{p-s}{m_p}}.
\]
When $b\leq k$, the numerator is zero and $L=0$. Otherwise,
\[
\frac{\binom{p-2k}{m_p}}{\binom{p-k}{m_p}}
=\prod_{i=1}^k\frac{b-i}{N-i}\leq q^k.
\]
All denominators are positive because $p\geq2k$. Combining this with
the bound on $\sum_{h=1}^k a_h$ proves the last inequality in
\eqref{eq:finite_training_weights}.

Write $t=B_{k,p}/A_{k,p}$ and let $G$ denote the expected training error gap on the
left side of \eqref{eq:training_bound}. The first bound in
\eqref{eq:finite_training_weights} gives
\[
\sum_{j=1}^k d_j^2\leq qD.
\]
Also, the weights are nonnegative, so
\[
\sum_{j=k+1}^{2k}w_j(2-w_j)
\geq2(D-L)-D^2.
\]
It follows from \eqref{eq:training_bound}--\eqref{eq:training_lower} that
\begin{equation}
\label{eq:finite_training_objective}
\frac{G}{A_{k,p}}\geq D(2t-q-tD)-2tL.
\end{equation}

We now choose
\begin{equation}
\label{eq:ea_subst}
N=p-k+1,\qquad
b=\left\lfloor\frac{NB_{k,p}}{3A_{k,p}}\right\rfloor,\qquad
m_p^\star=N-b=\left\lceil N\left(1-\frac{B_{k,p}}{3A_{k,p}}\right)\right\rceil.
\end{equation}
For fixed $p,k$, these quantities depend only on the averages $A_{k,p},B_{k,p}$ of expected squared coefficients
and are fixed before observing the response. Since $p\geq10/t$
and $p\geq2k$, we have $N\geq p/2+1$, $N\geq6$ and
$1/N\leq t/5$. Also $Nt\geq5+t>5$, so $1\leq b<N$ and
$m_p^\star\geq2N/3>1$. Thus $2\leq m_p^\star\leq p-k$, and all the
preceding bounds apply with $m_p=m_p^\star$.

Put $e=1/N$ and $x=q/t$. Since $Nt/3>1$ and
$\lfloor y\rfloor\geq y/2$ for $y\geq1$, the floor in \eqref{eq:ea_subst} gives
\[
\frac16\leq x\leq\frac13.
\]
Moreover,
\[
D\leq\frac{q}{1-q-e}
\leq\frac{xt}{1-(x+1/5)t}\leq\frac57.
\]
The function $D\mapsto D(2t-q-tD)$ is increasing for
$D\leq1-x/2$, and
\[
1-\frac x2\geq\frac56>\frac57.
\]
Using $D\geq q$ and the bound on $L$ in
\eqref{eq:finite_training_objective} therefore gives
\[
\frac{G}{A_{k,p}}\geq2tq-(1+t)q^2-
\frac{2tq^{k+1}}{1-q-e}.
\]
When $k\geq2$, we have $q^{k+1}\leq q^3$, and hence
\[
\frac{G}{B_{k,p}^2/A_{k,p}}\geq H_t(x),\qquad
H_t(x)=2x-(1+t)x^2-
\frac{2x^3t^2}{1-(x+1/5)t}.
\]
For fixed $x\in[1/6,1/3]$, this expression is decreasing in
$t\in(0,1]$. Thus $H_t(x)\geq H_1(x)$. The function
\[
H_1(x)=2x-2x^2-\frac{2x^3}{4/5-x}
\]
is concave on $(0,4/5)$, since
$H_1''(x)=-256/[125(4/5-x)^3]<0$. Its minimum on
$[1/6,1/3]$ is therefore attained at an endpoint, where
\[
H_1(1/6)=\frac5{19}>\frac14,\qquad
H_1(1/3)=\frac27>\frac14.
\]
This proves the claimed bound when $k\geq2$.

For $k=1$, the same choice has $N=p$ and $pt\geq10$.
Since $\lfloor y\rfloor\geq3y/4$ for $y\geq3$, we have
$x\in[1/4,1/3]$. The exact weights are
\[
\weight{1}{1,m_p^\star,p}=1-q,\qquad
\weight{2}{1,m_p^\star,p}=\frac{p}{p-1}q(1-q)\geq q(1-q).
\]
Keeping only the reward at rank two and using that $u(2-u)$ is increasing
on $[0,1]$, the expected training error gap satisfies
\[
\frac{G}{B_{k,p}^2/A_{k,p}}
\geq2x-x^2-3x^2t+2x^3t^2-x^4t^3.
\]
The right side is decreasing in $t$, with derivative
$-x^2(3-4xt+3x^2t^2)<0$. At $t=1$, it equals
$g(x)=2x-4x^2+2x^3-x^4$. This function is concave, and
\[
g(1/4)=\frac{71}{256}>\frac14,\qquad
g(1/3)=\frac{23}{81}>\frac14.
\]
The claimed bound follows for the single deterministic estimator
EFS($k,m_p^\star$).
\end{proof}

\subsection{Constant regression function plus Gaussian noise}
\label{subsec:constant_regression}

Suppose $y_i=\mu+\varepsilon_i$, where $\mu$ is constant and $\varepsilon_1,\ldots,\varepsilon_n$ are IID $\mathcal N(0,\sigma^2)$ noise variables.
Assume the constant component is orthogonal to the feature dictionary, $\langle\mathbf{1}_n,\bx_{\cdot j,p}\rangle=0$ for all $j$, as happens after fitting an intercept or centering the feature columns.
Consequently, the fixed design mean vector $\bbf=\mu\mathbf{1}_n$ has no projection onto the feature directions.
Under Assumption~\ref{ass:orthogonal}, the unsorted OLS coefficients $\langle\boldsymbol{\varepsilon},\bx_{\cdot j,p}\rangle$ are independent $\mathcal N(0,\sigma^2/n)$ random variables.
After sorting them by decreasing absolute value as in Definition~\ref{def:betaj}, standard Gaussian extreme value asymptotics imply that, for each fixed $r$,
\[
\E\big[\hat\beta_{(r),p}^2\big]=(1+o(1))\frac{2\sigma^2\log p}{n},
\qquad p\to\infty.
\]
Thus, for fixed $k$ and $\sigma>0$, along any sequence $n=n_p\geq p+1$, we have $A_{k,p}/B_{k,p}\to1$ and the condition in Theorem~\ref{thm:training} holds for all sufficiently large $p$.
The theorem gives the more interpretable bound
\[
\E\big[\|\mathbf{y}-\hat{\mathbf{f}}_{k,m_p^\star,p}\|^2\big]
\leq\E\big[\|\mathbf{y}-\hat{\mathbf{f}}_{k,p}\|^2\big]
-(1+o(1))\frac{\sigma^2\log p}{2n}.
\]
The comparable expected squared coefficients near the selection cutoff allow EFS to improve the fit by spreading the weight budget across more coefficients.

\subsection{Dependence on the expected squared coefficients}
\label{subsec:training_ratio_necessity}

\begin{proposition}
\label{prop:training_ratio_necessity}
For every $c>2$, there is a sequence of centered, normalized orthogonal
designs with $n=p+1$, deterministic mean vectors, and IID
$\mathcal N(0,1)$ noise such that, writing
\[
A_{1,p}=\E[\hat\beta_{(1),p}^2],\qquad
B_{1,p}=\E[\hat\beta_{(2),p}^2],
\]
we have $A_{1,p}/(pB_{1,p})\to c$, but
\[
\E\big[\|\by-\hat\bbf_{1,m,p}\|^2\big]
>\E\big[\|\by-\hat\bbf_{1,p}\|^2\big]
\]
for every $p\geq2$ and $1\leq m<p$.
\end{proposition}

\begin{proof}
First suppose a fixed response has squared OLS coefficients
$A_0\geq B_0>0$ and zero otherwise. For $m<p$, put $r=(p-m)/p$.
The weights after one step satisfy
\[
\weight{1}{1,m,p}=1-r,\qquad
\weight{2}{1,m,p}=\frac{m}{p-1}r\leq r.
\]
Since $u(2-u)$ increases on $[0,1]$, \eqref{eq:training_decomp}
gives the training error improvement
\[
\begin{split}
\Delta_m
&=B_0\weight{2}{1,m,p}\big(2-\weight{2}{1,m,p}\big)-A_0r^2\\
&\leq B_0r(2-r)-A_0r^2\\
&\leq-r^2\{A_0-(2p-1)B_0\}.
\end{split}
\]
Here we used $r\geq1/p$. Thus every $m<p$ has
$\Delta_m\leq-[A_0-(2p-1)B_0]/p^2<0$ when $A_0/B_0>2p-1$.

Fix $c>2$ and $p\geq2$ for the noisy construction. Take $n=p+1$ and normalized orthogonal
features in the orthogonal complement of the constant vector.
Let $\bbeta^{(0)}=(\sqrt{cp},1,0,\ldots,0)$ and use mean vector
$\bbf=\lambda\sum_{j=1}^p\beta_j^{(0)}\bx_{\cdot j,p}$ with IID
$\mathcal N(0,1)$ noise. The empirical coefficients satisfy
\[
\frac{\hat\beta_{j,p}}\lambda
=\beta_j^{(0)}+\frac{Z_j}{\lambda\sqrt n},
\]
where $Z_1,\ldots,Z_p$ are IID standard normal variables.
Sorting squared coordinates is continuous. For $\lambda\geq1$, their
sum is bounded by
$2(cp+1)+(2/n)\sum_{j=1}^p Z_j^2$, which is integrable.
The absolute value of the training error difference in
\eqref{eq:training_decomp}, divided by $\lambda^2$, is bounded by
this same sum. Dominated convergence therefore makes each expected
training error difference, divided by $\lambda^2$, tend to the strictly
negative value above with $A_0=cp$, $B_0=1$. It also gives
$A_{1,p}/B_{1,p}\to cp$ as $\lambda\to\infty$ for this fixed $p$.
There are only $p-1$ candidate sizes. Thus we may choose a deterministic
$\lambda_p$ large enough that every expected difference is negative
and $|A_{1,p}/B_{1,p}-cp|\leq1/p$. This proves the proposition.
Since $A_{1,p}/B_{1,p}\sim cp$, a condition $p\geq C h(A_{1,p}/B_{1,p})$ with
$h(t)=o(t)$ would hold eventually for these models, despite the absence
of any improvement.
\end{proof}

\subsection{Weak signals}
\label{subsec:weak_signals}

\begin{lemma}[Weak signals]
\label{lem:weak_signal_second_moment}
Fix $\sigma>0$. Grant Assumption~\ref{ass:orthogonal}
along any integer sequence with $p\to\infty$ and $n=n_p\geq p$.
Let $k=k_p\geq1$ be integers with $k=o(p)$.
For each $p$, let the mean vector be deterministic, suppose the noise
is Gaussian, and write
$s_p=|\{j:\beta_{j,p}\ne0\}|$ and $\tau_p=\sigma/\sqrt{n_p}$.
Assume that, for all sufficiently large $p$ with $s_p\geq1$,
\[
\max_{1\leq j\leq p}|\beta_{j,p}|
\leq\tau_p\max\left\{1,\sqrt2\left(\sqrt{\log p}
-\sqrt{\log s_p}-\sqrt{\log\log p}\right)\right\}.
\]
Then, for every fixed integer $r\geq1$,
\begin{equation}
\label{eq:weak_signal_rank_second_moment}
\frac{\E[\hat\beta_{(r),p}^2]}{2\tau_p^2\log p}\longrightarrow1.
\end{equation}
Moreover, writing
\[
A_{k,p}=\frac1k\sum_{r=1}^k\E[\hat\beta_{(r),p}^2],
\qquad
B_{k,p}=\frac1k\sum_{r=k+1}^{2k}\E[\hat\beta_{(r),p}^2],
\]
for $p\geq2k$, we have $A_{k,p}\sim B_{k,p}\sim2\tau_p^2\log(p/k)$, and hence $A_{k,p}/B_{k,p}\to1$.
\end{lemma}

\begin{proof}
By orthogonality and Gaussian noise, we can write
$\hat\beta_{j,p}=\tau_p(Z_j+\mu_{j,p})$, where $Z_1,\ldots,Z_p$ are
independent standard normal variables and
$\mu_{j,p}=\beta_{j,p}/\tau_p$.
Let $T_j=|Z_j+\mu_{j,p}|$, and let $T_{(1)}\geq\cdots\geq T_{(p)}$
denote their decreasing order. Write $s=s_p$, $M=\max_j|\mu_{j,p}|$,
$L=\log(p/k)$, and $a=\sqrt{2L}+1$.

We first bound the sum of the largest $k$ squared coefficients. When
$s\geq1$, the coefficient bound gives
$M\leq1+\sqrt{2\log p}-\sqrt{2\log s}$ for all sufficiently large $p$.
If $s\geq k$, then
\[
\sqrt{\log p}-\sqrt{\log s}
\leq\sqrt{\log(p/k)}-\sqrt{\log(s/k)},
\]
so $M+\sqrt{2\log(s/k)}\leq a$. Thus, for $x\geq0$, a Gaussian tail
bound gives
\[
\sum_{j:\mu_{j,p}\ne0}\mathbb P(T_j>a+x)
\leq2s e^{-(a+x-M)^2/2}\leq2k e^{-x^2/2}.
\]
If $1\leq s<k$, put $d=\sqrt{2\log(k/s)}$. Since
$\sqrt{\log p}-\sqrt{\log s}\leq\sqrt{\log(p/s)}$, we have $M\leq a+d$.
Using $u_+=\max\{u,0\}$, we obtain
\[
\begin{split}
\sum_{j:\mu_{j,p}\ne0}\mathbb P(T_j>a+x)
&\leq2s e^{-(x-d)_+^2/2}\\
&=2k e^{-(d^2+(x-d)_+^2)/2}
\leq2k e^{-x^2/4},
\end{split}
\]
where the last inequality follows from $d^2+(x-d)_+^2\geq x^2/2$.
The zero coefficients contribute at most
$2p e^{-(a+x)^2/2}\leq2k e^{-x^2/2}$. Consequently, in all cases,
including $s=0$,
\[
\Lambda(a+x)=\sum_{j=1}^p\mathbb P(T_j>a+x)
\leq4k e^{-x^2/4},\qquad x\geq0.
\]
Since
\[
\sum_{r=1}^kT_{(r)}^2
\leq ka^2+\sum_{j=1}^p(T_j^2-a^2)_+,
\]
integrating the tail bound gives
\[
\begin{split}
\frac{A_{k,p}}{\tau_p^2}
&\leq a^2+\frac1k\int_0^\infty2(a+x)\Lambda(a+x)\,dx\\
&\leq a^2+8a\sqrt\pi+16
=2L+O(\sqrt L+1),
\end{split}
\]
uniformly over $s$ and the signal coefficients.

For the lower bound, let $V_{1,p}\geq\cdots\geq V_{p,p}$ denote the
decreasing order of $Z_1^2,\ldots,Z_p^2$.
Let $Z$ be standard normal with distribution function $\Phi$.
For every real $\mu$, the magnitude $|Z+\mu|$ stochastically dominates
$|Z|$. Indeed, for $t,\mu\geq0$, the derivative in $\mu$ of
$\Phi(t-\mu)+\Phi(t+\mu)-1$ is nonpositive. Independence and a
coordinatewise quantile coupling therefore give
\[
\E[T_{(r)}^2]\geq\E[V_{r,p}],\qquad 1\leq r\leq p.
\]
Fix $0<\epsilon<1$ and put $b=\sqrt{2(1-\epsilon)L}$.
The number $N_b$ of indices with $|Z_j|>b$ is binomial with mean
$\lambda=p\mathbb P(|Z|>b)$. For $b\geq1$, integrating the normal
density over $[b,b+1/b]$ gives
$\mathbb P(|Z|>b)\geq c b^{-1}e^{-b^2/2}$ for an absolute constant
$c>0$. Therefore
\[
\frac{\lambda}{k}
\geq\frac{c e^{\epsilon L}}{\sqrt{2(1-\epsilon)L}}
\longrightarrow\infty.
\]
As $k\geq1$, we also have $\lambda\to\infty$. The Chebyshev inequality
and $\operatorname{Var}(N_b)\leq\lambda$ now give
$\mathbb P(N_b\geq2k)\to1$. Since $2k\leq p$ eventually,
\[
\frac{B_{k,p}}{\tau_p^2}
\geq\E[T_{(2k)}^2]\geq\E[V_{2k,p}]
\geq b^2\mathbb P(N_b\geq2k)
=2(1-\epsilon)L(1-o(1)).
\]
Combining the bounds with $B_{k,p}\leq A_{k,p}$, using $L\to\infty$, and then
letting $\epsilon\downarrow0$ gives
$A_{k,p}\sim B_{k,p}\sim2\tau_p^2\log(p/k)$, and hence $A_{k,p}/B_{k,p}\to1$.
For each fixed $r$, taking $k=r$ in the bounds and using
\[
T_{(2r)}^2\leq T_{(r)}^2\leq\frac1r\sum_{j=1}^rT_{(j)}^2
\]
also proves \eqref{eq:weak_signal_rank_second_moment}, since
$\log(p/r)/\log p\to1$.

For $p\geq2k$, Gaussian noise with $\tau_p>0$ gives $B_{k,p}>0$.
Since $A_{k,p}/B_{k,p}\to1$, the condition $p\geq10A_{k,p}/B_{k,p}$ in
Theorem~\ref{thm:training} holds for all sufficiently large $p$.
Lemma~\ref{lem:low_snr_concavity} gives concavity under the same
coefficient bound, including when every projected signal coefficient is zero.
Corollary~\ref{cor:prediction} therefore gives strictly smaller expected
prediction error for the same deterministic choice $m_p^\star$.
\end{proof}

\subsection{An interior minimum of expected training error}
\label{subsec:nonmonotone_training}
Theorem~\ref{thm:training} establishes improvement over $m=p$, but does not by itself exclude a minimum at $m=1$.
Set $p=n=3$, $k=1$, and take an orthogonal design with mean coefficients $(2,1,1/\sqrt{10})$.
At zero noise, the rank weights for $m=1,2,3$ are respectively $(1/3,1/3,1/3)$, $(2/3,1/3,0)$, and $(1,0,0)$.
Thus the respective training errors are
\[
 \frac{34}{15},\qquad \frac{89}{90},\qquad \frac{11}{10},
\]
so $m=2$ is a strict interior minimum.
This example also holds with positive Gaussian noise variance sufficiently small.
Indeed, couple the responses as $\by_\sigma=\bbf+\sigma\bz$ with $\bz\sim\mathcal N(\boldsymbol{0},\bI_3)$.
The strict ordering of the mean coefficient magnitudes makes each fitted vector continuous at $\sigma=0$.
Moreover, for $0\leq\sigma\leq1$, orthogonality and weights in $[0,1]$ give
$\|\by_\sigma-\hat{\bbf}_{1,m,3}\|^2\leq\|\by_\sigma\|^2\leq2\|\bbf\|^2+2\|\bz\|^2$.
Dominated convergence therefore preserves both strict inequalities in expectation for all sufficiently small $\sigma>0$.

\subsection{Proof of Lemma \ref{lem:low_snr_concavity}}
\label{subsec:proof_low_snr_concavity}

\begin{proof}
Write $\tau=\sigma/\sqrt n$. In this proof, feature indices refer to the
fixed, original dictionary, before the ordering determined by the response.
Set $\mu_j=\beta_{j,p}/\tau$ and
$X_j=\hat\beta_{j,p}/\tau$.
Orthogonality and Gaussian noise imply that the $X_j$ are independent
$\mathcal N(\mu_j,1)$ variables. Let $T_j=|X_j|$, and write $I_k=\pi(k)$ for the
original index of the $k$th largest magnitude, so that
$T_{I_1}>\cdots>T_{I_p}$ almost surely.
The covariance definition of degrees of freedom gives
\[
    \Delta_k = \df(\hat f_{k,p})-\df(\hat f_{k-1,p})
    =\E[X_{I_k}(X_{I_k}-\mu_{I_k})].
\]
The sign of $X_j$, conditionally on $T_j=t$, satisfies
$\E[X_j\mid T_j=t]=t\tanh(\mu_jt)$. Consequently,
\[
    \Delta_k=\sum_{j=1}^p\E[h_j(T_j)\,\mathbf{1}(I_k=j)],
    \qquad
    h_j(t)=t^2-\mu_jt\tanh(\mu_jt).
\]
Here and below $h_j$ depends only on $|\mu_j|$.
The density of $T_j$ on $(0,\infty)$ is
\[
    f_j(t)=\sqrt{\frac2\pi}\exp\!\left(-\frac{t^2+\mu_j^2}{2}\right)
           \cosh(\mu_jt),
    \qquad
    h_j(t)=-t\,\frac{f_j'(t)}{f_j(t)}.
\]

When $p=2$, integration by parts gives, for $t\geq0$,
\[
 \E[X_j(X_j-\mu_j)\mathbf{1}(T_j>t)]
 =-\int_t^\infty u f_j'(u)\,du
 =\mathbb P(T_j>t)+t f_j(t).
\]
Conditioning on the other magnitude and summing over $j=1,2$ therefore gives
\[
 \df(\hat f_{1,2})=1+2\int_0^\infty t f_1(t)f_2(t)\,dt>1.
\]
Since $\df(\hat f_{0,2})=0$ and $\df(\hat f_{2,2})=2$, this proves
concavity for $p=2$ without any restriction on the means.
For the rest of the proof, assume $p\geq3$.

We first establish an inequality in two variables. Fix $a,b\geq0$ and define
\[
    H(t,u)=\cosh(at)\cosh(bu)+\cosh(bt)\cosh(au),
    \qquad L(t,u)=\log H(t,u).
\]
For $t\geq u\geq0$, we claim that
\begin{equation}
\label{eq:pair_folded_normal}
    tL_t(t,u)-uL_u(t,u)\leq\frac{a^2+b^2}{2}(t^2-u^2).
\end{equation}
To see this, put $v=(t+u)/2$, $r=(t-u)/2$, $A=a+b$, and $B=|a-b|$.
The addition formula for the hyperbolic cosine gives
\[
    \widetilde H(v,r) = H(v+r,v-r)
    =\cosh(Av)\cosh(Br)+\cosh(Bv)\cosh(Ar).
\]
Let $\omega=\cosh(Av)\cosh(Br)/\widetilde H(v,r)\in[0,1]$.
Since $\tanh(z)\leq z$ for $z\geq0$ and $v,r,A,B\geq0$,
\[
\begin{split}
    tL_t-uL_u
    &=r(\log\widetilde H)_v+v(\log\widetilde H)_r\\
    &=\omega\{rA\tanh(Av)+vB\tanh(Br)\}\\
    &\quad +(1-\omega)\{rB\tanh(Bv)+vA\tanh(Ar)\}\\
    &\leq vr(A^2+B^2)
     =\frac{a^2+b^2}{2}(t^2-u^2).
\end{split}
\]
This proves \eqref{eq:pair_folded_normal}.

Now fix $k\in\{1,\ldots,p-1\}$ and distinct labels $i,j$. Given
$t>u>0$, let $Q_{ij,k}(t,u)$ be the probability that exactly $k-1$
of the other $p-2$ magnitudes exceed $t$ and all the remaining
$p-k-1$ magnitudes are below $u$. Thus $Q_{ij,k}\geq0$, and it is
the same whether label $i$ takes magnitude $t$ or magnitude $u$.
The contribution of this unordered pair to $\Delta_k-\Delta_{k+1}$ is the
integral over $t>u>0$ of $Q_{ij,k}(t,u)$ times
\[
\begin{split}
    S_{ij}(t,u)
    ={}&[h_i(t)-h_j(u)]f_i(t)f_j(u)\\
       &+[h_j(t)-h_i(u)]f_j(t)f_i(u).
\end{split}
\]
Writing $F_{ij}=f_i(t)f_j(u)+f_j(t)f_i(u)$, the identity for $h_j$
implies
\[
    S_{ij}=(-t\partial_t+u\partial_u)F_{ij}.
\]
With $a=|\mu_i|$ and $b=|\mu_j|$, $F_{ij}$ is a positive constant
times $\exp(-(t^2+u^2)/2)H(t,u)$. Hence
\[
\begin{split}
    S_{ij}(t,u)
    &=F_{ij}(t,u)\{t^2-u^2-tL_t(t,u)+uL_u(t,u)\}\\
    &\geq F_{ij}(t,u)\left(1-\frac{\mu_i^2+\mu_j^2}{2}\right)(t^2-u^2).
\end{split}
\]
Summing over unordered pairs and integrating against the nonnegative
$Q_{ij,k}$ gives $\Delta_k\geq\Delta_{k+1}$ whenever
$\max_j|\mu_j|\leq1$. The same argument applies whenever the sum of the two largest squared mean magnitudes is at most $2$.

Write $M=\max_j|\mu_j|$, $s=|\{j:\mu_j\ne0\}|$, and $N=p-s$.
The preceding argument covers $M\leq1$, including $s=0$, and also
$M\leq\sqrt2$ when $s=1$.
Under \eqref{eq:sparse_concavity_bound}, this proves the result for
$3\leq p\leq e^5$ at every sparsity, and for $3\leq p\leq e^9$ when $s\geq2$.
Indeed, the function $L\mapsto\sqrt{2L}-\sqrt{2\log L}$ decreases and then
increases on $L>1$, since its derivative has the sign of $L\log L-1$.
Its value at $L=1$ is $\sqrt2$, and
\[
\sqrt{10}-\sqrt{2\log5}<\sqrt2,\qquad
\sqrt{18}-\sqrt{2\log9}<1+\sqrt{2\log2}.
\]
Checking the endpoints on $[1,5]$ and $[1,9]$ therefore gives the two claims.
For the remaining cases, $p>e^5$ and $M>1$.
For mean magnitude $v\geq0$, write the folded normal density and distribution functions as
\[
 f(t;v)=\phi(t-v)+\phi(t+v),\qquad
 F(t;v)=\mathbb P(|Z+v|\leq t),\qquad Q(t;v)=1-F(t;v),
\]
where $Z$ is standard normal and $\phi,\Phi$ are its density and distribution function.
Put $\rho_v(t)=vt\tanh(vt)$ and $\overline\Phi=1-\Phi$.
With $G_k=T_{I_k}^2-T_{I_{k+1}}^2$ and $A_k=\E[G_k]$, we have
\begin{equation}
\label{eq:concavity_exact_gap}
 \Delta_k-\Delta_{k+1}=A_k-C_k,\qquad
 C_k=\E\big[\rho_{|\mu_{I_k}|}(T_{I_k})
          -\rho_{|\mu_{I_{k+1}}|}(T_{I_{k+1}})\big].
\end{equation}
The pairwise argument bounds this bias difference by
\begin{equation}
\label{eq:concavity_pair_correction}
 C_k\leq\frac12\E\big[G_k(\mu_{I_k}^2+\mu_{I_{k+1}}^2)\big].
\end{equation}
The same bound holds after restricting both expectations to any region
that depends only on the two adjacent magnitudes.

We first compare $A_k$ with $A_k^{(0)}(p)$, the expected gap between the $k$th and $(k+1)$st largest squared magnitudes when all means are zero.
The density ratio of a noncentral squared normal to a central one is
$e^{-\mu^2/2}\cosh(\mu\sqrt x)$, which increases in $x$. Its survival
hazard is therefore at most the central $\chi_1^2$ hazard $h_0(x)$.
Moreover,
\[
 h_0'(x)=h_0(x)\left\{h_0(x)-\tfrac12-\tfrac1{2x}\right\}\leq0
\]
by $\overline\Phi(t)\geq\phi(t)t/(1+t^2)$.
Conditional on the $(k+1)$st largest squared value being $v$ and on
the identities of the $k$ variables above it, these variables are
independent truncations. Their minimum excess over $v$ has expectation at least
\[
 m_k(v)=\int_0^\infty
 \left\{\frac{Q(\sqrt{v+x};0)}{Q(\sqrt v;0)}\right\}^{k}\,dx.
\]
Decreasing central hazard makes $m_k(v)$ increasing in $v$.
Coordinatewise stochastic domination gives $A_k\geq A_k^{(0)}(p)$.
The same argument shows that $A_k^{(0)}(p)$ increases with $p$ for fixed $k$.

For a central folded normal, $U=F(T;0)$ is uniform and
$f(t;0)\leq\sqrt{2/\pi}$. Thus the inverse distribution function satisfies
\[
 [F^{-1}(u;0)]^2-[F^{-1}(v;0)]^2\geq\frac\pi2(u^2-v^2),\qquad 0\leq v\leq u<1.
\]
The smallest expected difference of squares of adjacent uniform order statistics
is $4/((p+1)(p+2))$. Using monotonicity in $p$ also gives
\begin{equation}\label{eq:concavity_spacing_dimension}
 A_k\geq\frac{2\pi}{(p+1)(p+2)}>p^{-2},\qquad
 A_k\geq\frac{2\pi}{(k+2)(k+3)}.
\end{equation}
For the first rank, $A_1^{(0)}(p)\geq A_1^{(0)}(2)=\E[|Z_1^2-Z_2^2|]=4/\pi$
for independent standard normals $Z_1,Z_2$, since $Z_1^2-Z_2^2=(Z_1-Z_2)(Z_1+Z_2)$
and the two factors are independent $\mathcal N(0,2)$ variables.

Suppose first that $s=1$. Let $\E_0$ denote expectation under $p$
independent central folded normals, with ordered magnitudes
$T^{(0)}_{(1)}\geq\cdots\geq T^{(0)}_{(p)}$. Put
\[
 r(t)=\frac{f(t;M)}{f(t;0)}=e^{-M^2/2}\cosh(Mt),\qquad
 J(t)=\rho_M(t)r(t)=e^{-M^2/2}Mt\sinh(Mt).
\]
Multiplying the central density by $r$ for the signal label and using
symmetry gives
\[
 C_k=\frac1p\E_0\big[J(T^{(0)}_{(k)})-J(T^{(0)}_{(k+1)})\big].
\]
For $P_k(t)=\binom pk Q(t;0)^kF(t;0)^{p-k}$, integration over the
interval between adjacent magnitudes yields
\[
 C_k=\frac1p\int_0^\infty J'(t)P_k(t)\,dt,\qquad
 A_k^{(0)}(p)=\int_0^\infty2tP_k(t)\,dt.
\]
The function $J'(t)/(2t)$ increases, as follows from its power series.
The probability densities proportional to $2tP_k(t)$ have successive
likelihood ratios proportional to $Q(t;0)/F(t;0)$, which decrease.
Since the covariance of an increasing and a decreasing function is
nonpositive, we obtain
\begin{equation}\label{eq:concavity_single_rank}
 \frac{C_k}{A_k^{(0)}(p)}\leq\frac{C_1}{A_1^{(0)}(p)},
 \qquad 1\leq k<p.
\end{equation}
It therefore suffices to bound the bias at rank one.

Write
\[
\begin{gathered}
 L=\log p,\quad \ell=\log L,\quad a=\sqrt{2L},\quad h=\sqrt{2\ell},
 \\
 M_0=a-h,\quad t_0=\sqrt{2L-\ell},\quad d_0=t_0-M_0.
\end{gathered}
\]
Here $L>5$ and $1<M\leq M_0<t_0$. Since $P_1(t)=pQ(t;0)F(t;0)^{p-1}$,
\[
 C_1=\int_0^\infty J'(t)Q(t;0)F(t;0)^{p-1}\,dt.
\]
For $t\leq t_0$, $J'(t)/(2t)\leq M^2r(t_0)$. This part of the
integral is at most $cA_1^{(0)}(p)$, where
\[
 c=\frac{M^2r(t_0)}p
 \leq(1+e^{-2t_0})\frac{M_0^2}{2L}e^{-M_0(a-t_0)}.
\]
To see the inequality, use $M>1$ in
$r(t_0)=(1+e^{-2Mt_0})e^{Mt_0-M^2/2}/2$ and note that
$M^2e^{Mt_0-M^2/2}$ increases for $M<t_0$.
Put
\[
 x=\sqrt{\ell/L},\qquad y=\sqrt{1-x^2/2},\qquad z=1-x,
 \qquad v=1-\frac{x}{2(1+y)}.
\]
Then $M_0=az$, $t_0=ay$, and $d_0=hv$. For $L\geq5$,
\[
 x<\frac35,\quad y>\frac9{10},\quad v>\frac56,\quad
 \ell\geq3z,\quad \frac z{\sqrt\ell}<\frac25,\quad
 e^{\ell(1-v^2)}<2.
\]
These bounds require only elementary calculus. Indeed,
$x=\sqrt\ell e^{-\ell/2}$ decreases for $\ell>1$.
It exceeds $1/2$ on $[3/2,2]$ and $1/3$ on $[2,3]$, proving
$\ell\geq3z$ there, and the case $\ell\geq3$ is immediate.
The function $\ell^{-1/2}-e^{-\ell/2}$ is less than $2/5$ at $\ell=1$
and tends to zero at infinity. At a stationary point its value is
$(\ell-1)/\ell^{3/2}\leq2/(3\sqrt3)<2/5$.
Finally, $1-v^2\leq x/2$ gives
$\ell(1-v^2)\leq3\sqrt3/(2e^{3/2})<\log2$.
Consequently,
\[
 c<\frac54z^2e^{-\ell z/2}
 \leq\frac54z^2e^{-3z^2/2}\leq\frac5{6e}<\frac13.
\]

For $t>t_0$, Mills' bound $Q(t;0)\leq f(t;0)/t$ gives
$J'(t)Q(t;0)\leq(M^2+M/t_0)f(t;M)$.
Since $Q(t_0;M)$ increases with $M$, this part of the integral is at most
\[
 \begin{split}
 U&=(M_0^2+M_0/t_0)Q(t_0;M_0)\\
 &\leq(M_0^2+M_0/t_0)(1+e^{-2t_0})\frac{\phi(d_0)}{d_0}\\
 &<\frac54\frac{z}{\sqrt\ell}
       \frac{e^{\ell(1-v^2)}}{\sqrt\pi v}
       \left(z+\frac1{2Ly}\right)\\
 &<\frac54\frac25\frac{2}{\sqrt\pi(5/6)}
       \left(1+\frac19\right)=\frac4{3\sqrt\pi}.
 \end{split}
\]
Here the reflected tail is at most $e^{-2M_0t_0}$ times the other tail,
and $M_0>1$. Therefore, combining the two integration regions,
\[
 \frac{C_1}{A_1^{(0)}(p)}<\frac13+\frac{\sqrt\pi}{3}<1.
\]
Together with \eqref{eq:concavity_single_rank} and $A_k\geq A_k^{(0)}(p)$,
this proves concavity when $s=1$.

Now suppose $s\geq2$. The elementary cases above allow us to assume
$L=\log p>9$ and $M>1$. Put
\[
 a=\sqrt{2L},\qquad h=\sqrt{2\log L},\qquad
 b=\sqrt{2\log s},\qquad d=h+b.
\]
The coefficient bound gives $M\leq a-d$. In particular,
\[
 a>4,\qquad h>2,\qquad b>\frac76,\qquad d>3,
 \qquad e^{-bh}<\frac1{10}.
\]
Since $b<a-3$, we have $s/p<e^{-a}$. Hence
$N-2>9p/10$, and Mills' bound gives $F(M;M)>9/20$.

We first bound the probability of a signal label at either endpoint.
Conditional on adjacent magnitudes $t>u>0$,
\[
\begin{gathered}
 \E\big[\mathbf{1}(\mu_{I_k}\ne0)+\mathbf{1}(\mu_{I_{k+1}}\ne0)\mid t,u\big]
 \leq\frac{2s}{N-2}\max\{1,R_M(t,u)\},
 \\
 R_M(t,u)=\max_{0<v\leq M}\frac{f(t;v)}{f(t;0)}\frac{F(u;0)}{F(u;v)}.
\end{gathered}
\]
To see this, swap a signal endpoint with a null label outside the pair.
If that null is above $t$, the ratio of the original conditional density to that after the exchange is at most one by the increasing likelihood ratio $f(t;v)/f(t;0)$.
If it is below $u$, the ratio is at most $R_M(t,u)$.
Each signal endpoint has at least $N-2$ outgoing swaps, and each
resulting configuration has at most $2s$ incoming swaps. Summing their
conditional weights proves the bound.

Split adjacent magnitudes $t\geq u$ into three regions,
\[
 \begin{split}
 \mathcal R_1&=\{t\leq a,\ (u\geq M\text{ or }t-u\leq1/M)\},\\
 \mathcal R_2&=\{t\leq a,\ u<M,\ t-u>1/M\},\\
 \mathcal R_3&=\{t>a\}.
 \end{split}
\]
For $u\geq M$ and $t\leq a$,
$R_M(t,u)\leq e^{Ma-M^2/2}/F(M;M)$.
For $u<M$ and $t-u\leq1/M$, use
\[
 \partial_t\log\frac{f(t;v)}{f(t;0)}\leq v,\qquad
 \frac{f(u;v)}{F(u;v)}\leq v+1/u,\qquad
 \frac{F(u;0)}{f(u;0)}\leq ue^{u^2/2}.
\]
The middle inequality follows by integrating $\partial_t\log f(t;v)\leq v$.
These bounds give
\[
 R_M(t,u)\leq e(1+M^2)e^{M^2/2}\leq e^{Ma-M^2/2},
\]
because $1+\log(1+M^2)\leq2M<M(a-M)$ for $M\geq1$.
By \eqref{eq:concavity_pair_correction}, the contribution of $\mathcal R_1$
to $C_k$ is therefore at most
\[
 \frac{M^2s}{N-2}\frac{e^{Ma-M^2/2}}{F(M;M)}A_k
 <\frac{40}{81}A_k<\frac12A_k.
\]
Here we used $M^2<2L$ and
$(s/p)e^{Ma-M^2/2}\leq e^{-bh}/L$.

An interval of length $1/M$ starting below $M$ has central folded normal
mass at least
\[
 j(M)=\frac{\sqrt{2/\pi}}M
 e^{-(M+1/(2M))^2/2}>\frac3{8M}e^{-M^2/2}.
\]
Indeed, its smallest mass starts at $M$, and the Jensen inequality applies
because $\phi$ is convex on $[1,\infty)$.
On $\mathcal R_2$, no null lies in this interval just above $u$.
Conditioning on the label of $u$ and taking a union bound gives
$\mathbb P(\mathcal R_2)\leq p e^{-K}$, where
\[
 K=(N-1)j(M)>\frac{27}{80a}e^{3a/2}>2a^2=4L.
\]
The last inequality holds because $e^{3a/2}/a^3$ increases for $a\geq4$
and $e^6>400$. The exact bias on this region is at most $Ma<2L$.
Using \eqref{eq:concavity_spacing_dimension}, its contribution is less than
\[
 2Lp e^{-K}<2Le^{-L}A_k<\frac1{100}A_k.
\]

For the last region, put
\[
 D=2\phi(h)e^{-bh}<\frac1{10L},\qquad
 \Lambda=\frac2{\sqrt{2\pi}a}+\frac Dd<\frac14+\frac1{270}<1.
\]
For $w\geq0$, the bounds
$f(a+w;v)\leq2\phi(d+w)$ and $Q(a+w;v)\leq2\overline\Phi(d+w)$
for $v\leq M$ imply
\[
 \sum_{j:\mu_j\ne0}f(a+w;|\mu_j|)\leq De^{-dw},\qquad
 \lambda(a+w)=\sum_{j=1}^p Q(a+w;|\mu_j|)\leq\Lambda e^{-dw}.
\]
For the first inequality, use $s\phi(h+b+w)=\phi(h)e^{-bh-dw-w^2/2}$.
For the second, apply Mills' bound to the signals and use
$p\phi(a)=1/\sqrt{2\pi}$ and $a>d$ for the nulls.
The exact bias on $\mathcal R_3$ is at most
$MT_{I_k}\mathbf{1}(\mu_{I_k}\ne0)$.
Independence and the bound using factorial moments for the number of other
magnitudes above $t$ therefore bound its expectation by
\[
 \begin{split}
 \frac{MD\Lambda^{k-1}}{(k-1)!}\int_0^\infty(a+w)e^{-kdw}\,dw
 &\leq MD\left(\frac ad+\frac1{d^2}\right)\frac{\Lambda^{k-1}}{k!}\\
 &<\frac{13}{180k!}<\frac1{12k!}<\frac16A_k.
 \end{split}
\]
The last inequality follows from \eqref{eq:concavity_spacing_dimension}
and $(k+2)(k+3)\leq12k!$ for every $k\geq1$.
The three regional bounds sum to less than $A_k$, so
\eqref{eq:concavity_exact_gap} proves concavity for every $1\leq k<p$.

\end{proof}

\subsection{Proofs for Section \ref{subsec:dof_main}}
\label{subsec:dof_appendix}

\begin{proof}[Proof of Theorem \ref{thm:increasingm}]
For every response, the representation in terms of rank weights gives
\[
\hat f_{k,m,p}
=\sum_{j=1}^p\weight{j}{k,m,p}
 (\hat f_{j,p}-\hat f_{j-1,p}).
\]
The weights are deterministic, so linearity of covariance proves
\eqref{eq:df_alt_decomp}.

We claim that whenever $m\leq m'$, the sequence $\big\{w_j^{k,m',p}\big\}_{j=1}^p$ majorizes the sequence $\big\{w_j^{k,m,p}\big\}_{j=1}^p$. This means that
\begin{equation}
\label{eq:CDF_comparison}
\sum_{j=1}^l\weight{j}{k,m',p}
\geq\sum_{j=1}^l\weight{j}{k,m,p},\qquad 1\leq l\leq p,
\end{equation}
with equality when $l=p$. Here all indices denote ranks. Let $S_{r,m}$ denote
the ranks of the features selected by one randomized FS path after $r$ steps, and put
$Z_{r,m}(l)=|S_{r,m}\cap[l]|$. Then
$\sum_{j=1}^l\weight{j}{k,m,p}=\E[Z_{k,m}(l)]$.

We couple the randomized FS paths with parameters $m$ and $m'$ so that $Z_{r,m'}(l)\geq Z_{r,m}(l)$ at every step $r$ and every cutoff $l$.
At step $r+1$, there are $p-r$ remaining features in each path. Set $a=\min\{m,p-r\}$ and $a'=\min\{m',p-r\}$.
Draw a uniform subset $B'$ of $[p-r]$ with $a'$ elements, then a uniform subset
$B\subseteq B'$ with $a$ elements. Both marginal subsets are uniform of their respective
sizes. Use their elements as positions in the increasing lists of remaining
ranks for the two paths. This construction also implements the cap on the candidate size when fewer than $m$ or $m'$ features remain. Write $b=\min B$ and
$b'=\min B'$, so $b'\leq b$. The selected rank is the $b$th remaining
rank on the first path, and therefore
\[
Z_{r+1,m}(l)=Z_{r,m}(l)+\mathbf{1}(b\leq l-Z_{r,m}(l)),
\]
with the analogous identity for $m'$ and $b'$.

We now show the desired property inductively in $r$, with the base case $r=0$ trivial.
Assume the inductive hypothesis $Z_{r,m'}(l)\geq Z_{r,m}(l)$ for every $l$.
If the inequality is strict at a given $l$, its gap is a positive integer,
so adding at most one to either count preserves the ordering.
If the counts agree, the numbers $l-Z_{r,m'}(l)$ and
$l-Z_{r,m}(l)$ of remaining ranks up to the cutoff agree, and $b'\leq b$ ensures that the increment
along the $m'$ path is at least that along the $m$ path. Thus
$Z_{r+1,m'}(l)\geq Z_{r+1,m}(l)$ simultaneously for all $l$.
Taking expectations at step $k$ proves \eqref{eq:CDF_comparison}.
The total weights of both paths equal $k$.

Finally, put $D_j=\df(\hat f_{j,p})$, $\Delta_j=D_j-D_{j-1}$, and
$A_l=\sum_{j=1}^l(w_j^{k,m',p}-w_j^{k,m,p})$. By the preceding comparison,
$A_l\geq0$ and $A_p=0$. Concavity of $D_j$ means
$\Delta_l\geq\Delta_{l+1}$. Summation by parts now gives
\[
\begin{split}
\df(\hat f_{k,m',p})-\df(\hat f_{k,m,p})
&=\sum_{j=1}^p\Delta_j(w_j^{k,m',p}-w_j^{k,m,p})\\
&=\sum_{l=1}^{p-1}(\Delta_l-\Delta_{l+1})A_l\geq0.
\end{split}
\]
No sign assumption on the increments $\Delta_j$ is needed.
\end{proof}

\begin{proof}[Proof of Corollary \ref{cor:prediction}]
The degrees of freedom comparison follows from Theorem \ref{thm:increasingm}, since $m_p^\star<p$.
Combining it with \eqref{eq:error_gap} in \eqref{eq:expected_error} gives the strict prediction error improvement.
\end{proof}

\begin{proof}[Proof of Proposition \ref{prop:non_monotonic_df}]
For the fixed original ordering of the features, let $\tau=\sigma/\sqrt n$.
Throughout this proof, $k,\sigma,\gamma$ are fixed, but $n=n_p$ and
$m=m_p$ may follow any sequences allowed by the proposition. Since we
consider the limit as $p$ goes to infinity, we may assume $p>k$.
By orthogonality and Gaussian noise, we have independent standard normal
variables $Z_1,\ldots,Z_p$ such that
\[
\hat\beta_{i,p}=\beta_{i,p}+\tau Z_i,\qquad
\beta_{i,p}=\mathbf{1}(i\leq k).
\]
Let $I_r=\pi(r)$ be the original label of the coefficient at rank $r$,
and define
\[
Q_i=\tau^{-2}\hat\beta_{i,p}(\hat\beta_{i,p}-\beta_{i,p})
=Z_i^2+\tau^{-1}\beta_{i,p}Z_i.
\]
Writing $w_r=\weight{r}{k,m,p}$, by linearity of covariance, we can then write
\begin{equation}
\label{eq:df_fixed_coordinate_sum}
D_k\coloneqq \df(\hat f_{k,p})=\E\Biggl[\sum_{r=1}^kQ_{I_r}\Biggr],\qquad
D_{k,m}\coloneqq \df(\hat f_{k,m,p})=\E\Biggl[\sum_{r=1}^p w_rQ_{I_r}\Biggr].
\end{equation}

Let $\mathcal A$ be the event that every signal coefficient is positive
and its magnitude exceeds every null coefficient magnitude. For
$i\leq k<j$, both $\hat\beta_{i,p}-\hat\beta_{j,p}$ and
$\hat\beta_{i,p}+\hat\beta_{j,p}$ have distribution
$\mathcal N(1,2\tau^2)$. Since
\[
\{|\hat\beta_{j,p}|>|\hat\beta_{i,p}|\}
\subseteq\{\hat\beta_{j,p}>\hat\beta_{i,p}\}
\cup\{-\hat\beta_{j,p}>\hat\beta_{i,p}\},
\]
using Gaussian tail bounds and applying a union bound, we get
\begin{equation}
\label{eq:df_support_failure}
\delta_p\coloneqq \mathbb P(\mathcal A^c)
\leq2k(p-k)e^{-n/(4\sigma^2)}+ke^{-n/(2\sigma^2)}.
\end{equation}
To control the contribution from $\mathcal A^c$, let $R_p=\sum_{i=1}^p|Q_i|$. By the triangle inequality in $L^2$, we have
\[
\begin{aligned}
(\E[R_p^2])^{1/2}
&\leq\left\{\E\left[\left(\sum_{i=1}^p Z_i^2\right)^2\right]\right\}^{1/2}
 +\tau^{-1}\left\{\E\left[\left(\sum_{i=1}^k|Z_i|\right)^2\right]\right\}^{1/2}
\\
&\leq\sqrt{p(p+2)}+k\sqrt n/\sigma\eqqcolon B_p.
\end{aligned}
\]
Applying Cauchy--Schwarz, any two random variables bounded in absolute
value by $R_p$ and equal on $\mathcal A$ have expectations differing by at
most $\eta_p\coloneqq 2B_p\sqrt{\delta_p}$. Using $p\leq n$, we get
$\eta_p\leq C_{k,\sigma}n^{3/2}e^{-n/(8\sigma^2)}$, and hence
$\sup_{n\geq p}\eta_p\to0$, uniformly over $m$. In particular,
\begin{equation}
\label{eq:df_support_comparison}
|D_k-k|\leq\eta_p,
\end{equation}
because $\sum_{r=1}^kQ_{I_r}=\sum_{i=1}^kQ_i$ on $\mathcal A$ and
$\E\big[\sum_{i=1}^kQ_i\big]=k$.

First suppose $k\geq2$. Write $Z_{[1]}\geq\cdots\geq Z_{[k]}$ for the
decreasing order of the signal noises $Z_1,\ldots,Z_k$, and write
$U_{1,p}\geq\cdots\geq U_{p-k,p}$ for the decreasing order of the null
squares $Z_{k+1}^2,\ldots,Z_p^2$. On $\mathcal A$, the first $k$ ranks
are ordered by the $Z_{[r]}$, and therefore
\[
\sum_{r=1}^p w_rQ_{I_r}
=\tau^{-1}\sum_{r=1}^k w_rZ_{[r]}
 +\sum_{r=1}^k w_rZ_{[r]}^2
 +\sum_{r=1}^{p-k}w_{k+r}U_{r,p}.
\]
Notice that both sides are bounded in absolute value by $R_p$ on the
whole probability space, since each is a sum of the same $Q_i$ with
weights in $[0,1]$, possibly in a different order. Hence,
\begin{equation}
\label{eq:df_signal_order_expansion}
D_{k,m}=\frac{\sqrt n}{\sigma}a_{k,p}+b_{k,p}+T_{k,p}+O(\eta_p),
\end{equation}
where
\[
a_{k,p}=\sum_{r=1}^k w_r\E[Z_{[r]}],\qquad
b_{k,p}=\sum_{r=1}^k w_r\E[Z_{[r]}^2],\qquad
T_{k,p}=\sum_{r=1}^{p-k}w_{k+r}\E[U_{r,p}].
\]
We have $0\leq b_{k,p}\leq k$. Since the weights sum to $k$,
Lemma~\ref{lem:max_chi_squared} also gives
\[
0\leq T_{k,p}\leq k\E\big[\max_{k<i\leq p}Z_i^2\big]
\leq k\{2\log(2(p-k))+2\}=O_k(\log p).
\]
Theorem~\ref{thm:asymp} implies
$a_{k,p}\to a_{k,\gamma}\coloneqq \sum_{r=1}^k\limitweight{r}{k,\gamma}\E[Z_{[r]}]$.

We now show that this limit is strictly positive. Let $q=1-\gamma\in(0,1)$
and $u_j^r=\limitweight{j}{r,\gamma}$.
In the infinite rank process from the proof of
Lemma~\ref{lem:finite_population_weights}, choosing $G>j$ for the first
$k-1$ steps and $G=j$ next selects rank $j$ for the first time, with
probability $\gamma q^{jk-1}>0$. Thus $u_j^k-u_j^{k-1}>0$ for every $j$.
Combining \eqref{eq:r1} and \eqref{eq:r2} gives
\[
 u_j^k-u_{j+1}^k=(1-q^k)(u_j^k-u_j^{k-1})>0.
\]
Since $\sum_{r=1}^k\E[Z_{[r]}]=0$, we have
\[
 a_{k,\gamma}
 =\frac1k\sum_{1\leq i<j\leq k}
 (u_i^k-u_j^k)\E[Z_{[i]}-Z_{[j]}]>0.
\]
The last inequality follows because the ordered Gaussian noises are
distinct almost surely.

Putting \eqref{eq:df_support_comparison} and
\eqref{eq:df_signal_order_expansion} together and using $n\geq p$, we get
\begin{equation}
\label{eq:df_sqrt_n_limit}
\frac{\sigma}{\sqrt n}(D_{k,m}-D_k)\longrightarrow a_{k,\gamma}>0.
\end{equation}
Since $\sqrt n/\log p\to\infty$, the claimed logarithmically normalized
limit is $+\infty$ for $k\geq2$.

Now suppose $k=1$. Let $S$ be the uniform candidate set of size $m$, drawn
independently of the response, and let $J(S)$ denote the selected feature.
By linearity of covariance, $D_{1,m}=\E_{S,Z}[Q_{J(S)}]$. On $\mathcal A$,
we have $Q_{J(S)}=Q_1$ if $1\in S$, and
$Q_{J(S)}=\max_{i\in S}Z_i^2$ otherwise.
By a similar calculation using $R_p$, the error in this replacement is at
most $\eta_p$.
Conditional on $1\notin S$, the $m$ sampled noises are IID standard
normal. Thus, writing $M_m$ for the maximum of $m$ IID $\chi^2_1$ variables,
\[
D_{1,m}=\frac mp+\left(1-\frac mp\right)\E[M_m]+O(\eta_p),\qquad
D_{1,m}-D_1=\left(1-\frac mp\right)(\E[M_m]-1)+O(\eta_p).
\]
Here $\E[M_m]/\log m\to2$ by Lemma~\ref{lem:max_chi_squared}, and
$\log m/\log p\to1$ because $m/p\to\gamma\in(0,1)$.
The ratio of the last display to $\log p$ therefore tends to
$2(1-\gamma)$.

Finally, $\df(\hat f_{k,1,p})=k$ and
$\df(\hat f_{k,p,p})=D_k=k+o(1)$. In either case,
a subsampling ratio converging to a fixed interior value gives degrees of freedom larger than
both endpoints for all sufficiently large $p$. This finishes the proof.
\end{proof}

\begin{remark}[A uniform comparison for one step]
For $k=1$, the comparison holds for every $2\leq m\leq p-1$, once
$p$ is sufficiently large depending only on $\sigma$, uniformly over
$n\geq p$. Indeed, the comparison using candidate sets in the proof has error
at most $2\eta_p$ in the difference of degrees of freedom, independently
of $m$. Its bound satisfies
\[
p\eta_p\leq C_\sigma n^{5/2}e^{-n/(8\sigma^2)},
\qquad
\sup_{n\geq p}p\eta_p\longrightarrow0.
\]
For independent standard normals $X,Y$,
\[
\E[\max\{X^2,Y^2\}]
=1+\tfrac12\E[|(X-Y)(X+Y)|]=1+\frac2\pi,
\]
because $X-Y$ and $X+Y$ are independent $\mathcal N(0,2)$ variables,
each with mean absolute value $2/\sqrt\pi$. Thus $\E[M_m]-1\geq2/\pi$
for every $m\geq2$. Since $1-m/p\geq1/p$ when $m\leq p-1$, the same
comparison gives
\[
\df(\hat f_{1,m,p})-\df(\hat f_{1,p})
\geq\frac{2}{\pi p}-2\eta_p>0
\]
uniformly over the stated range for all sufficiently large $p$.
\end{remark}

\begin{remark}
In this proposition, $k,\sigma$ are fixed and $m/p$ tends to a fixed
$\gamma\in(0,1)$.
Under orthogonality, if $1\leq k\leq p$ and $m>p-k$, every base learner
selects exactly the first $k$ ranks. Indeed, at most $p-k$ features lie
outside those ranks, so every candidate set contains a remaining feature among the first $k$
ranks until all are selected. A capped candidate set contains all
remaining features. Thus $\hat f_{k,m,p}=\hat f_{k,p}$ for every response
in this range. In particular, $m=p-1$ gives exact equality for $k\geq2$,
whereas the preceding uniform strict comparison concerns only $k=1$.
At $k=0$ or $k=p$, all candidate sizes give the
same estimator.
\end{remark}

\begin{lemma}
\label{lem:max_chi_squared}
Let $Z_1,Z_2,\ldots,Z_N$ be IID $\chi^2_1$ variables, and let
$M_N=\max_{1\leq i\leq N}Z_i$. We have
\[
\E[M_N]\leq2\log(2N)+2,\qquad
\frac{\E[M_N]}{\log N}\longrightarrow2\quad(N\to\infty).
\]
In particular, $\E[M_N]\geq\frac12\log N$ for all sufficiently large $N$.
\end{lemma}

\begin{proof}
Using the Gaussian tail bound, we have
$\mathbb P(M_N>t)\leq\min\{1,2Ne^{-t/2}\}$ for $t\geq0$.
Therefore,
\begin{equation}
\label{eq:max_chi_squared_helper}
\E[M_N]=\int_0^\infty\mathbb P(M_N>t)\,dt
\leq2\log(2N)+2.
\end{equation}
For the lower bound, let $\phi$ and $\Phi$ denote the standard Gaussian
density and distribution function. Integration by parts gives
$1-\Phi(z)\geq z\phi(z)/(1+z^2)$ for $z>0$.
Fix $\epsilon\in(0,2)$ and set $t_N=(2-\epsilon)\log N$. Then
\[
N\mathbb P(Z_1>t_N)
=2N\{1-\Phi(\sqrt{t_N})\}
\geq\sqrt{\frac2\pi}\frac{\sqrt{t_N}}{1+t_N}N^{\epsilon/2}
\longrightarrow\infty.
\]
Independence implies
$\mathbb P(M_N>t_N)\geq1-\exp\{-N\mathbb P(Z_1>t_N)\}\to1$.
As such, $\E[M_N]\geq t_N\mathbb P(M_N>t_N)$, so the liminf of
$\E[M_N]/\log N$ is at least $2-\epsilon$. Letting $\epsilon\downarrow0$
and using the upper bound completes the proof.
\end{proof}

\subsection{Proof of Proposition \ref{thm:shrinkage}}
\label{subsec:proof_of_shrinkage}

\begin{proof}
According to the decomposition \eqref{eq:decomp_ensemble}, it suffices to
prove \eqref{ineq:square} below. Fix the response and write
$a_j=\hat\beta_{(j),p}^2$, so that $a_1\geq\cdots\geq a_p\geq0$.
For $m\leq m'$, let
$u_j=w_j^{k,m,p}$, $v_j=w_j^{k,m',p}$, and
$S_l=\sum_{j=1}^l(v_j-u_j)$, with $S_0=0$.
By \eqref{eq:CDF_comparison}, we have $S_l\geq0$, with $S_p=0$
because both weight sequences sum to $k$.
By Lemma~\ref{lem:weights_monotonicity}, the sequence
$a_j(u_j+v_j)$ is nonnegative and nonincreasing. Factoring the difference
of squares and summing by parts gives
\begin{equation}
\begin{split}
\sum_{j=1}^p a_j(v_j^2-u_j^2)
&=\sum_{j=1}^p a_j(u_j+v_j)(v_j-u_j)\\
&=\sum_{l=1}^{p-1}
 \big\{a_l(u_l+v_l)-a_{l+1}(u_{l+1}+v_{l+1})\big\}S_l
\geq0.
\end{split}
\label{ineq:square}
\end{equation}
Under orthogonality, the two sums are the squared Euclidean norms of the
corresponding coefficient vectors and the squared normalized norms of the
fitted vectors. This proves the
inequality, including the cases $k=0$, $k=p$, and $m=m'$.
\end{proof}

\subsection{Proof of Proposition \ref{prop:escape_local_min}}
\label{subsec:proof_escape_local_min}

\begin{proof}
For convenience, number the original features in the fixed tie order.
Thus ties are resolved by the smallest index.
For an active set $M$, the squared selection score of an eligible feature $j$ is
$|\langle\bx_{\cdot j,p},\bP_M^\perp\by\rangle|^2/\|\bP_M^\perp\bx_{\cdot j,p}\|^2$.
It equals the reduction in squared residual loss upon adding $j$.
A support $T$ of size $k$ is a strict local minimum under exchanges
of a single feature if $L(T)<L((T\setminus\{i\})\cup\{j\})$ for all
$i\in T$, $j\notin T$, where $L(M)=\|\by-\bP_M\by\|^2$.

For each $(n,p,k,m)$ in the stated domain, we construct the design and
noiseless response before drawing the candidate sets. We consider two cases.
First suppose $1\leq m\leq p-k$.
Take $n\geq p$ and vectors $\mathbf e_1,\ldots,\mathbf e_p$
orthonormal in the normalized inner product. Let $S\geq6$ be a parameter
to be chosen below, and define
\[
\by_0=\mathbf e_1+\mathbf e_2+\sqrt{S-2}\,\mathbf e_3,
\qquad a=\frac{3}{4S},\qquad s=\frac9{16}.
\]
Use the features
\[
\bx_{\cdot 1,p}=\mathbf e_1,\qquad \bx_{\cdot 2,p}=\mathbf e_2,\qquad
\bx_{\cdot 3,p}=a\by_0+\sqrt{1-s/S}\,\mathbf e_4,\qquad
\bx_{\cdot 4,p}=a\by_0-\sqrt{1-s/S}\,\mathbf e_4,
\]
and $\bx_{\cdot i,p}=\mathbf e_i$ for $5\leq i\leq p$. Write
$H=\{5,\ldots,k+2\}$ for the $k-2$ padding features, with $H$ empty
when $k=2$, and set
\[
\by=\by_0+2\sum_{i\in H}\mathbf e_i.
\]
Features indexed by $\{k+3,\ldots,p\}$ have zero response coefficient.
Every feature has norm one. Since $S>2$ and $1-s/S>0$, the first four
features span $\mathbf e_1,\ldots,\mathbf e_4$. The remaining features are
mutually orthogonal and orthogonal to this span, so the dictionary has
full column rank.

The padding features have score $2$ until selected, the first two
features initially have score $1$, the two good features $3,4$ have
score $3/4$, and the null features have score zero. The padding features
are orthogonal to the other features, so FS selects all of $H$ first.
After either feature $1$ or $2$ is selected, the squared score of
either good feature is
\[
g_1(S)=\frac{s(1-1/S)^2}{1-s/S^2}<s<1,
\]
whereas the other member of $\{1,2\}$ has squared score one.
Consequently FS($k$) selects
$T=H\cup\{1,2\}$, with residual loss $S-2$.

We now show that $T$ is a strict local minimum. A feature outside $T$
is either one of $3,4$ or a null feature. If a feature in $\{1,2\}$
is replaced by a good feature, the loss increases by $1-g_1(S)>0$.
Replacement by a null feature increases it by one. If a padding
feature is removed, the loss increases by $4$.
Adding a good feature to the remaining first two features reduces the loss by
only
\[
g_2(S)=\frac{s(1-2/S)^2}{1-2s/S^2}<s<4,
\]
so this exchange increases the loss by
$4-g_2(S)>0$. Adding a null feature captures nothing and increases
the loss by four. The inequalities $g_1(S)<s$ and $g_2(S)<s$ follow
by comparing their numerators and denominators for $S\geq6$.
Thus every exchange of a single feature increases the loss, so $T$ is a strict
local minimum. On the other hand, $H\cup\{3,4\}$ fits the response exactly,
because $\bx_{\cdot 3,p}+\bx_{\cdot 4,p}=2a\by_0$.

We now study the base estimator of EFS, denoted $\check f_{k,m,p}$, which
depends on a random seed $\theta$. We construct a sequence of candidate
sets that selects all padding features first, then feature $3$, then
feature $4$.
For each padding step, choose any fixed candidate set containing
the remaining padding feature with the smallest index. This feature is selected
because its score is two and all scores outside the padding set are smaller. After
these $k-2$ steps, there are $p-k+2$ eligible features. Exactly $p-k$ of
them lie outside $\{1,2\}$. The hypothesis $m\leq p-k$ therefore
allows a fixed candidate set of size $m$ containing feature $3$
and excluding features $1,2$. Its other members are feature $4$
and/or null features. It selects feature $3$.

At the last step, choose a fixed candidate set containing feature $4$.
After feature $3$ is selected, adding feature $4$ eliminates the remaining
core loss $S-s$. Each of features $1,2$ instead has squared score
\[
\frac{(1-s/S)^2}{1-s/S^2}<1<S-s,
\]
and the null features have score zero. Thus feature $4$ is selected,
and the final loss is zero. By the score inequalities above, the same
sequence of candidate sets gives this selection order for every $S\geq6$.

At step $l\leq k$, at least $p-k+1>m$ features remain. Hence the
probability of this particular path of candidate sets is
\[
\rho=\prod_{l=1}^{k}\binom{p-l+1}{m}^{-1}>0,
\]
independently of $S$. The fit along every path is an orthogonal projection,
so its residual loss is at most
$\|\by\|^2=S+4(k-2)$. It follows that
\[
\E_{\theta}\big[\|\by-\check{\bbf}_{k,m,p}\|^2\big]
\leq(1-\rho)\{S+4(k-2)\}.
\]
Choose
\[
S>\max\left\{6,\frac{2+4(k-2)(1-\rho)}{\rho}\right\}.
\]
Then the last upper bound is strictly smaller than $S-2$, the FS loss.
Using the Jensen inequality, we also have
\[
\|\by-\hat{\bbf}_{k,m,p}\|^2
\leq\E_{\theta}\big[\|\by-\check{\bbf}_{k,m,p}\|^2\big]<S-2,
\]
which proves the result in this case.

Now suppose $p-k+1\leq m\leq p-1$. The preceding path can no longer exclude both features $1,2$ after selecting the padding features. For fixed $(n,p,k)$, we therefore use a different construction that works for every $m$ in this range.
Retain the orthonormal vectors and the padding index set $H$, but put
\[
\begin{gathered}
\by_0=3\mathbf e_1+\mathbf e_2+\mathbf e_3,\qquad
\bx_{\cdot 1,p}=\mathbf e_1,\quad \bx_{\cdot 2,p}=\mathbf e_2,\\
\bx_{\cdot 3,p}=\frac{2\by_0+\sqrt{77}\mathbf e_4}{11},\qquad
\bx_{\cdot 4,p}=\frac{2\by_0-\sqrt{77}\mathbf e_4}{11},
\end{gathered}
\]
with $\bx_{\cdot i,p}=\mathbf e_i$ for $i\geq5$, and set
\[
\by=\by_0+\frac12\sum_{i\in H}\mathbf e_i.
\]
The features are normalized and linearly independent. The initial core
scores are $3,1,2,2$, the padding scores are $1/2$, and the null scores
are zero. FS first selects feature $1$. The squared scores of feature
$2$ and either good feature are then $1$ and $16/85$, respectively,
so it selects feature $2$. Thereafter, either good feature has squared
score $4/81<1/4$, so all padding features are selected before either
good feature. Thus FS($k$) selects $T=H\cup\{1,2\}$ and has loss one.

For the four exchanges of a feature in $\{1,2\}$ for a good feature, we
have
\[
\|\by_0-\bP_{\{1,j\}}\by_0\|^2=\frac{154}{85}>1,
\qquad
\|\by_0-\bP_{\{2,j\}}\by_0\|^2=\frac{770}{117}>1,
\qquad j\in\{3,4\}.
\]
For example, after feature $1$ the good feature adds $16/85$ to the squared fitted norm.
After feature $2$ it adds $400/117$. Replacing a padding feature by a
good feature loses $1/4$ in squared fitted norm and gains only $4/81$, so the loss
also increases. Replacing any selected feature by a null feature loses
its positive contribution to the squared fitted norm. Therefore $T$ is again a strict
local minimum under every exchange of a single feature.

Write $q=p-k-2$ for the number of null features. We first show that no
null feature is selected during the first $k$ randomized steps. Notice
that there are at most two eligible core features with zero score at any stage. Initially
all four core scores are positive. After one core feature has been
selected, the core residual is nonzero, so at least one of the three
remaining core features has positive score. After two or more core
selections, at most two core features remain. Consequently at most
$q+2$ eligible features have zero score. Also, an exact fit is
impossible before step $k$, because the unique representation in the full rank
dictionary uses both good features and all $k-2$ padding features.
Hence a feature with positive score remains before every step under
consideration. At each step $l\leq k$, the actual candidate size is
$\min\{m,p-l+1\}\geq p-k+1=q+3$. Every candidate set therefore contains
a feature with positive score, including when the cap on the candidate size applies.
Therefore, no null feature is selected.

Next, we show that the final support cannot contain exactly one good
feature. At the last step all $p-k+1=q+3\leq m$ eligible features are candidates. If one
good feature has already been selected, its partner captures the
entire remaining core residual. Projection loss is nonincreasing as features are
added, so the squared norm of that residual is at least its value after both bad
features and one good feature have been selected, namely $1-4/81=77/81$.
The partner therefore beats every padding feature, whose squared
score is $1/4$, and every other core feature, which captures strictly
less of the remaining core residual. To see the latter strict inequality,
note that by full rank, the unique
representation of $\by_0$ uses both good features with nonzero
coefficients. Hence a support missing one cannot fit $\by_0$ exactly.
The partner must therefore be selected at the last step.

It remains to exclude the possibility that the first good feature is
selected only at the last step. Before that step the selected set is
$T\setminus\{i\}$ for some $i\in T$. Since $T$ is a strict local minimum,
adding $i$ reduces the loss more than adding any feature outside $T$.
The full candidate set therefore selects $i$, which excludes this possibility.

Every final support consequently either contains no good feature,
in which case it is $T$ and has loss one, or contains both good
features. In the latter case its core loss is zero. If it has $r$
core features in total, it omits $r-2$ padding features, so its loss
is $(r-2)/4\leq1/2$, because $2\leq r\leq4$.

To obtain strict improvement, note that the first candidate set contains feature $3$ and excludes
feature $1$ with probability
\[
\frac{m(p-m)}{p(p-1)}>0.
\]
On this event feature $3$ is selected because its score two beats the other
available scores, with the index tie rule choosing it if feature $4$
is present. The final support then contains both good features and
has loss at most $1/2$. Putting these together, we have
\[
\E_\theta\big[\|\by-\check{\bbf}_{k,m,p}\|^2\big]
\leq1-\frac{m(p-m)}{2p(p-1)}<1.
\]
This is strictly below the FS loss. Using the Jensen inequality again gives
strictly smaller error for the infinite ensemble. The two cases cover
every $1\leq m<p$, which completes the proof.
\end{proof}

\begin{remark}[An example with four features]
For the first construction, take $p=4$, $k=2$, $m\in\{1,2\}$,
and $S=6$.
It gives $\by=\mathbf e_1+\mathbf e_2+2\mathbf e_3$ and
$\bx_{\cdot 3,p}=(\by+\sqrt{58}\mathbf e_4)/8$,
$\bx_{\cdot 4,p}=(\by-\sqrt{58}\mathbf e_4)/8$.
The FS support $\{1,2\}$ has loss four, each of its four exchanges
of a single feature has loss $290/63>4$, and $\{3,4\}$ has loss zero.
For $m=2$, the final support is $\{1,2\}$, $\{3,4\}$, or a cross pair
with probabilities $5/9$, $1/9$, and $1/3$, respectively. The first
candidate pair contains a member of $\{1,2\}$ with probability $5/6$,
and the second candidate pair contains the partner of the selected feature
with probability $2/3$. That partner has the largest score, as in the
proof. Thus the mean base learner loss is $710/189<4$.
For $m=1$, the six supports with two features are equally likely, giving
mean loss $706/189<4$.
For $m=3$, however, this particular design returns the FS support on every path.
The first candidate set contains feature $1$ or $2$, and the second
contains all remaining features. For $m=3$, we instead use the second
construction in the proof.
\end{remark}

\begin{remark}
The exclusions are necessary for strict improvement in the
mean base learner loss.
At $k=1$ or $k=p-1$, every competing
support of size $k$ can be obtained by exchanging a single feature, so a local minimum is
global and no individual base learner can improve it. This does not, however, rule
out an ensemble improvement, since an average of fits using $k$ features need not itself
be a fit using $k$ features. At $k=0$ every fit is zero, at $k=p$ every full rank
path gives the same fit, and at $m=p$ randomization disappears.
\end{remark}

\subsection{A finite guarantee for a fixed correlated family}
\label{subsec:finite_correlated_guarantee}

We now consider a different construction, in which randomization avoids
a poor greedy choice. The FS support in this construction is not a local
minimum. The following bound holds for all candidate sizes $2\leq m<p$
and gives the guarantee for a fixed proportion as a special case.

\begin{proposition}[A finite guarantee for all nontrivial candidate sizes]
\label{prop:broad_candidates}
Let $n\geq p\geq4$, $2\leq m\leq p-1$, and $3\leq k\leq p-1$ be integers.
Take normalized orthonormal vectors
$\bx_{\cdot1,p},\ldots,\bx_{\cdot\,p-1,p},\mathbf u$, put
$\bx_{\cdot p,p}=(\bx_{\cdot1,p}+\bx_{\cdot2,p}+\mathbf u)/\sqrt3$, and let
$\by=\beta(\bx_{\cdot1,p}+\bx_{\cdot2,p})+
\zeta\sum_{j=3}^{p-1}\bx_{\cdot j,p}$, where $\beta/\sqrt6>\zeta>0$.
Resolve ties by the smallest original feature index. Define
\[
\delta=1-\frac mp,
\qquad K_{m,p}=\max\{3,p-m+1\},
\]
\[
b_{k,m,p}=\begin{cases}
0,&k\geq K_{m,p},\\
\min\{1,2\mathbb P(\operatorname{Bin}(k,m/p)\leq2)\},&k<K_{m,p}.
\end{cases}
\]
Then the single randomized base learner satisfies
\begin{equation}
\|\by-\hat\bbf_{k,p}\|^2
-\E_\theta\big[\|\by-\check\bbf_{k,m,p}\|^2\big]
\geq \frac{\zeta^2\delta}{3}
-2(\beta^2-\zeta^2)b_{k,m,p}.
\label{eq:broad_candidates_gain}
\end{equation}
In particular, for every $2\leq m\leq p-1$ the gain is at least
$\zeta^2(p-m)/(3p)>0$ whenever $K_{m,p}\leq k\leq p-1$.
The gain is also strictly positive if
\begin{equation}
\frac{km}{p}\geq
2\log\left(\frac{96(\beta^2-\zeta^2)p}{\zeta^2(p-m)}\right).
\label{eq:broad_candidates_log_condition}
\end{equation}
Under this condition the gain is at least $\zeta^2(p-m)/(6p)>0$.
The infinite ensemble has no larger loss than the mean base learner loss.
\end{proposition}

\begin{proof}
Here, the randomness is over the choices of the candidate sets, with the
response fixed.
For an active set $M$ and an unselected feature $j$, define its selection score by
\[
s_j(M)=\frac{|\langle\bx_{\cdot j,p},\bP_M^\perp\by\rangle|}
{\|\bP_M^\perp\bx_{\cdot j,p}\|}.
\]
Since the construction has full column rank, all denominators are
positive. Each unselected distractor $3\leq j\leq p-1$ has score $\zeta$.
For an unselected signal feature $j\in\{1,2\}$, write $j'$ for the other signal
feature. We have
\begin{equation}
s_j(M)=
\begin{cases}
\beta,&p\notin M,\\
\beta/\sqrt6,&p\in M,\ j'\notin M,\\
\beta/\sqrt2,&p\in M,\ j'\in M.
\end{cases}
\label{eq:escape_signal_scores}
\end{equation}
For $p\notin M$, the spurious feature has score
\begin{equation}
s_p(M)=\frac{(2-t)\beta}{\sqrt{3-t}}
=\begin{cases}
2\beta/\sqrt3,&t=0,\\
\beta/\sqrt2,&t=1,\\
0,&t=2,
\end{cases}
\qquad t=|M\cap\{1,2\}|.
\label{eq:escape_spurious_scores}
\end{equation}
These formulas follow because every distractor is orthogonal to the three
special features. For example, after selecting only feature $p$ among those
features, a signal feature has residual correlation $\beta/3$ and squared
residualized norm $2/3$. After selecting $p$ and the other signal, these are
$\beta/2$ and $1/2$. If $p$ is unselected, its residual correlation and squared
residualized norm are $(2-t)\beta/\sqrt3$ and $(3-t)/3$.

Putting \eqref{eq:escape_signal_scores} and
\eqref{eq:escape_spurious_scores} together, FS first selects $p,1,2$.
This means that its loss is $L_{\rm FS}=(p-k)\zeta^2$.
For $k\leq p-1$, exchanging its redundant feature $p$ for any unselected
distractor lowers the loss by $\zeta^2$. Thus this FS support is not a
local minimum under exchanges of a single feature, including at $k=p-1$.
We now study the randomized path. Every unselected signal has score at least $\beta/\sqrt6>\zeta$.
As long as a signal remains unselected, every remaining special feature in
$\{1,2,p\}$ also has score larger than the distractor score $\zeta$.
Before any special feature has been selected, $p$ outranks the two signals.
After one signal has been selected without $p$, the other signal outranks $p$.
After both signals have been selected, $p$ has zero score and, because
$m\geq2$, is not selected before the final step.

Let $T_j$ denote the selection time of original feature $j$, extending the path
through all $p$ steps, and define
\[
G=\bigl\{\max\{T_1,T_2\}<T_p\bigr\},\qquad
E_3=\bigl\{\max\{T_1,T_2\}>k\bigr\},\qquad
E_1=G\cap E_3^c.
\]
We first compute a lower bound for $\mathbb P(G)$.
For fixed nominal candidate size $m\geq2$ and $N\geq2$, let $H_N$ be the probability that
the remaining signal is selected before $p$, in a state containing that signal,
$p$, and $N-2$ distractors, when the other signal has already been selected.
For $N\geq3$, let $G_N$ be the probability that both signals are selected before $p$ in a
state containing all three special features and $N-3$ distractors, with none of
the special features yet selected. Previously selected distractors do not change
these scores or probabilities. If $N\leq m$, the candidate set is the full
remaining set, so $H_N=1$ and $G_N=0$ (where their respective states exist).
For $N>m$, conditioning on the first selected feature gives the first
recurrence below. The second holds when also $N\geq4$.
\begin{align}
H_N&=\frac mN+
\frac{(N-m)(N-m-1)}{N(N-1)}H_{N-1},
\label{eq:broad_H_recurrence}\\
G_N&=a_NH_{N-1}+d_NG_{N-1},
\label{eq:broad_G_recurrence}\\
d_N&=\frac{(N-m)(N-m-1)(N-m-2)}{N(N-1)(N-2)},
\qquad a_N=1-\frac mN-d_N.
\nonumber
\end{align}
Here $d_N$ is the probability of a candidate set containing only distractors, and $a_N$
is the probability that a signal is selected first, requiring exclusion of $p$.
The remaining boundary case without a cap is $m=2,N=3$, where direct
conditioning gives $G_3=1/3$.

At the first step whose candidate set contains either of the two
remaining special features, symmetry makes each equally likely to appear
alone. If both appear, the signal is selected. Thus $H_N\geq1/2$.
We next prove $G_N\geq(N-m)/(3N)$ for $N\geq\max\{m,3\}$.
For $m\geq3$, the base case $N=m$ has both sides zero. For $m=2$, use
$G_3=1/3\geq1/9$. For the induction step, let $r=N-1\geq m$ and
$\delta_N=(N-m)/N$. Substituting $H_{N-1}\geq1/2$ and the bound for
$G_{N-1}$ under the inductive assumption, we obtain
\begin{align*}
G_N-\frac{\delta_N}{3}
&\geq\frac{a_N}{2}
+\frac{d_N(N-m-1)}{3(N-1)}-\frac{\delta_N}{3}\\
&=\frac{\delta_N m\{(3m+1)r-2m(m+1)\}}
{6r^2(r-1)}\geq0.
\end{align*}
The last inequality follows from $r\geq m$, because the expression in braces
is at least $m(m-1)$. Thus
\begin{equation}
\mathbb P(G)=G_p\geq\frac{p-m}{3p}=\frac\delta3.
\label{eq:broad_avoidance_bound}
\end{equation}

To compute a bound for $\mathbb P(E_3)$, let us generate the candidate
sets as follows. Independently at each step, first draw a uniform
subset of size $m$ from the $p$ features. Retain its unselected members and fill
the vacancies uniformly from the remaining eligible features. When fewer than
$m$ features remain, use all of them. Conditional on the past, permutation
symmetry makes the completed set uniform among eligible candidate sets. Every
preliminary appearance of an unselected signal is retained. Only the other two
special features can outrank it, and each can be selected only once, so three
preliminary appearances force that signal to be selected. For each signal, the preliminary
appearance indicators are IID $\operatorname{Bernoulli}(m/p)$ across steps. A union bound gives
\[
\mathbb P(E_3)\leq
2\mathbb P\big(\operatorname{Bin}(k,m/p)\leq2\big).
\]
We can also bound the latest selection time of the signals. Suppose a
signal is still unselected after $k$ steps. Throughout those steps every remaining special feature outranks
every distractor. Therefore a distractor can be selected only from a
candidate set containing only distractors, which requires at least $m$ remaining distractors.
At most $\max\{p-m-2,0\}$ distractors can have been selected. At most two
special features have been selected, since a signal is missing. It follows that
$k\leq\max\{p-m,2\}$. Consequently $\mathbb P(E_3)=0$ for
$k\geq K_{m,p}$, and in all cases $\mathbb P(E_3)\leq b_{k,m,p}$.

We next compute the loss for each possible set of selected features to combine these probability bounds.
Write $r\in\{0,1,2\}$ for the number of selected signals and
$s\in\{0,1\}$ for selection of $p$. The loss is
\[
L_\theta=(p-k-3+r+s)\zeta^2+\beta^2 e_{r,s},
\]
where the residual signal factors are
\[
\begin{array}{c|ccc}
 &r=0&r=1&r=2\\ \hline
s=0&2&1&0\\
s=1&2/3&1/2&0
\end{array}
\]
(each entry multiplies $\beta^2$). On $E_1$, the gain is exactly $\zeta^2$.
On $E_3^c\setminus E_1$, it is zero. On $E_3$, $r\leq1$ and the table gives
$L_{\rm FS}-L_\theta\geq3\zeta^2-2\beta^2$, since
$\beta^2>6\zeta^2$. Therefore
\begin{align*}
L_{\rm FS}-\E_\theta[L_\theta]
&\geq\zeta^2\mathbb P(E_1)-(2\beta^2-3\zeta^2)\mathbb P(E_3)\\
&\geq\zeta^2\mathbb P(G)-2(\beta^2-\zeta^2)\mathbb P(E_3)\\
&\geq\frac{\zeta^2\delta}{3}-2(\beta^2-\zeta^2)b_{k,m,p}.
\end{align*}
This proves \eqref{eq:broad_candidates_gain} and the assertion for
$k\geq K_{m,p}$. Using the recurrences above, we may also replace
$\delta/3$ by $G_p\geq\delta/3$ throughout.

For the logarithmic sufficient condition, put $\mu=km/p$. For $\mu\geq2$,
the Chernoff bound gives
\[
\mathbb P\big(\operatorname{Bin}(k,m/p)\leq2\big)
\leq e^{-\mu}(e\mu/2)^2\leq4e^{-\mu/2}.
\]
The last inequality follows by maximizing $\mu^2e^{-\mu/2}$ at $\mu=4$.
Condition~\eqref{eq:broad_candidates_log_condition} implies $\mu>2$,
because $\beta^2/\zeta^2>6$, and gives
$16(\beta^2-\zeta^2)e^{-\mu/2}\leq\zeta^2\delta/6$.
Combining this with \eqref{eq:broad_candidates_gain} proves the assertion.
Finally, using the Jensen inequality, we have
$\|\by-\hat\bbf_{k,m,p}\|^2\leq\E_\theta[L_\theta]$.
\end{proof}

\begin{remark}
For fixed proportions, suppose
$0<a\leq m/p\leq b<1$. Then choosing
$k\geq (2/a)\log\{96(\beta^2-\zeta^2)/(\zeta^2(1-b))\}$ suffices,
uniformly in $p$, subject to $3\leq k\leq p-1$.
For $m=p-1$ and $3\leq k\leq p-1$, the mean base learner gain is exactly $\zeta^2/p$.
The first candidate set excludes $p$ with probability $1/p$, forcing both
signals to be selected before $p$. Otherwise $p$ is selected first. Both
signals are always selected within three steps.

For this family, $m=1$ is, however, strictly worse in mean base learner loss for
every $3\leq k\leq p-1$. Choosing one candidate at each step gives a uniform subset of size $k$.
Let $R$ be the number of omitted signal features. The squared norm of the residual signal
is at least $\beta^2R/3$ in every state, by the table above. Thus its expected
loss is at least
\[
\frac{p-k}{p}\left((p-3)\zeta^2+\frac23\beta^2\right),
\]
which exceeds $(p-k)\zeta^2$ because $\beta^2>6\zeta^2$ and $k<p$.
Even with $m\geq2$, a threshold on $k$ independent of $m,p$ cannot suffice
for this fixed family. Take $m=2$ and $p=k^3$, for example, as $k\to\infty$.
A union bound gives probability at most $6k/(p-k+1)\to0$ that any
special feature appears among the first $k$ candidate sets. Otherwise
all $k$ selections are distractors, and the difference between the FS loss and the base learner loss is
$3\zeta^2-2\beta^2<0$. From the table above, the loss difference on every
path is bounded by a constant depending only on $\beta,\zeta$. Hence the expected difference
tends to this same negative value.
\end{remark}

\begin{remark}[Using increasing inclusion probabilities]
Under the hypotheses of Proposition~\ref{prop:broad_candidates}, the failure
term can be sharpened by using the inclusion probability at each step in place of $m/p$. Let $B_1,\ldots,B_k$ be independent with
$B_l\sim\operatorname{Bernoulli}(a_l)$, where
$a_l=\min\{m/(p-l+1),1\}$. Replace $b_{k,m,p}$ by
\[
b^+_{k,m,p}=\begin{cases}
0,&k\geq K_{m,p},\\
\min\{1,2\mathbb P(B_1+\cdots+B_k\leq2)\},&k<K_{m,p}.
\end{cases}
\]
Bound~\eqref{eq:broad_candidates_gain} remains valid with $b^+_{k,m,p}$
in place of $b_{k,m,p}$, and $b^+_{k,m,p}\leq b_{k,m,p}$.
To see this, fix one signal $j$, and write $T_j$ for its selection time. At each step when $j$ is unselected, first
draw its inclusion indicator as $B_l$, then sample a uniform candidate set
conditional on that inclusion or exclusion. This gives the required
uniform distribution over candidate sets, since each eligible feature is included with probability
$a_l$. After $j$ has been selected, draw the subsequent $B_l$ independently
and ignore them when sampling candidates. Three inclusions before selection
force $j$ to be selected, so $\mathbb P(T_j>k)\leq
\mathbb P(B_1+\cdots+B_k\leq2)$. Applying this construction separately to
the two signals and then using a union bound gives the result.

Let
\[
\mu^+_{k,m,p}=\sum_{l=1}^k\min\left\{\frac{m}{p-l+1},1\right\}.
\]
For $k<K_{m,p}$, we have $k\leq p-m$, so no cap applies and
$\mu^+_{k,m,p}=m(\mathsf H_p-\mathsf H_{p-k})$, where
$\mathsf H_r=\sum_{i=1}^r1/i$. The Chernoff calculation applies to independent
Bernoulli variables with unequal probabilities as well. Hence the condition
\[
\mu^+_{k,m,p}\geq
2\log\left(\frac{96(\beta^2-\zeta^2)p}{\zeta^2(p-m)}\right)
\]
gives mean base learner gain at least $\zeta^2(p-m)/(6p)$. When $k<K_{m,p}$,
apply the preceding Chernoff argument with this mean. When $k\geq K_{m,p}$,
the stronger bound for this range already holds. Since
$\mu^+_{k,m,p}\geq km/p$, this condition includes the earlier sufficient one.

Finally, fix $m\geq2$ and $\beta/\sqrt6>\zeta>0$ independently of $p$.
For any fixed $c\in(0,1)$ satisfying
\[
-m\log(1-c)>2\log\left(\frac{96(\beta^2-\zeta^2)}{\zeta^2}\right),
\]
and any integer sequence $k_p/p\to c$, we have
$\mu^+_{k_p,m,p}\to-m\log(1-c)$. Therefore, the sufficient condition
holds for all $p\geq p_0$, where $p_0$ may depend on $m,\beta,\zeta,c$
and the sequence $k_p$. Such a fixed $c<1$ exists, so the result includes
model sizes whose limiting fraction is strictly below one.
\end{remark}

\end{document}